\documentclass{article}

\usepackage[preprint,nonatbib]{neurips_2020}
\usepackage[numbers]{natbib}
\usepackage[utf8]{inputenc} 
\usepackage[T1]{fontenc}    
\usepackage{url}            
\usepackage{booktabs}       
\usepackage{amsfonts}       
\usepackage{nicefrac}       
\usepackage{microtype}      
\usepackage{graphicx}
\usepackage{subfig}
\usepackage{amssymb,amsbsy,amsmath,amsfonts,amssymb,amscd}
\usepackage{amsthm}
\usepackage{comment}
\usepackage{verbatim}
\usepackage[dvipsnames]{xcolor}
\usepackage{enumitem}
\usepackage{algorithm,algpseudocode}
\usepackage{hyperref}       
\usepackage{setspace}
\usepackage{xcolor}
\newtheorem{theorem}{Theorem}[section]

\newtheorem{lemma}{Lemma}[section]
\newtheorem{assum}{Assumption}[section]

\newcommand{\EE}{\mathbb{E}}

\newcommand{\Cov}{\mathrm{Cov}}
\newcommand{\Var}{\mathrm{Var}}

\newcommand{\rv}{\mathrm{V}}

\newcommand{\wx}{\widetilde{x}}
\newcommand{\wv}{\widetilde{v}}
\newcommand{\ww}{\widetilde{w}}
\newcommand{\wrx}{\widetilde{\mathrm{X}}}
\newcommand{\wrv}{\widetilde{\mathrm{V}}}

\newcommand{\rd}{\,\mathrm{d}}

\title{Variance reduction for Random Coordinate Descent-Langevin Monte Carlo}

\author{%
  Zhiyan Ding\\
  Department of Mathematics\\
  University of Wisconsin-Madison\\
  Madison, WI 53706 \\
  \texttt{zding49@math.wisc.edu} \\
  \And
   Qin Li \\
  Department of Mathematics \\
  University of Wisconsin-Madison\\
  Madison, WI 53706 \\
  \texttt{qinli@math.wisc.edu} \\
}

\begin{document}

\maketitle

\begin{abstract}

Sampling from a log-concave distribution function is one core problem that has wide applications in Bayesian statistics and machine learning. While most gradient free methods have slow convergence rate, the Langevin Monte Carlo (LMC) that provides fast convergence requires the computation of gradients. In practice one uses finite-differencing approximations as surrogates, and the method is expensive in high-dimensions.

A natural strategy to reduce computational cost in each iteration is to utilize random gradient approximations, such as random coordinate descent (RCD) or simultaneous perturbation stochastic approximation (SPSA). We show by a counter-example that blindly applying RCD does not achieve the goal in the most general setting. The high variance induced by the randomness means a larger number of iterations are needed, and this balances out the saving in each iteration.

We then introduce a new variance reduction approach, termed Randomized Coordinates Averaging Descent (RCAD), and incorporate it with both overdamped and underdamped LMC. The methods are termed RCAD-O-LMC and RCAD-U-LMC respectively. The methods still sit in the random gradient approximation framework, and thus the computational cost in each iteration is low. However, by employing RCAD, the variance is reduced, so the methods converge within the same number of iterations as the classical overdamped and underdamped LMC~\cite{DALALYAN20195278,Cheng2017UnderdampedLM,dalalyan2018sampling}. This leads to a computational saving overall.
\end{abstract}

\section{Introduction}

Monte Carlo Sampling is one of the core problems in Bayesian statistics, data assimilation~\cite{Reich2011}, and machine learning \cite{MCMCforML}, with wide applications in atmospheric science \cite{FABIAN198117}, petroleum engineering~\cite{PES}, remote sensing~\cite{LiNewton_MCMC} and epidemiology~\cite{COVID_travel} in the form of inverse problems~\cite{OGinverse}, volume computation~\cite{Convexproblem}, and bandit optimization \cite{ATTS}.

Let $f(x)$ be a convex function that is $L$-gradient Lipschitz and $\mu$-strongly convex in $\mathbb{R}^d$. Define the target probability density function $p(x) \propto e^{-f}$, then $p(x)$ is a log-concave function. To sample from the probability distribution induced by $p(x)$ amounts to finding an $x\in\mathbb{R}^d$ (or a list of $\{x^i\in\mathbb{R}^d\}$) that can be regarded as i.i.d. (independent and identically distributed) drawn from the distribution.

There is vast literature on sampling, and proposed methods fall into a few different categories. Markov chain Monte Carlo (MCMC)~\cite{Robert2004} composes a big class of methods, including Metropolis-Hasting based MCMC (MH-MCMC) \cite{doi:10.1063/1.1699114,MCSH}, Gibbs samplers \citep{Geman1984,10.2307/2685208}, Hamiltonian Monte Carlo~\citep{Neal1993,DUANE1987216}, Langevin dynamics based methods~\cite{doi:10.1063/1.436415} (including both the overdamped Langevin~\cite{PARISI1981378,roberts1996,doi:10.1111/rssb.12183} and underdamped Langevin~\cite{10.5555/3044805.3045080,10.5555/2969442.2969566} Monte Carlo), and some kind of combination (such as MALA)~\cite{roberts1996,LDMHA2002,dwivedi2018logconcave,MALA3}. Importance sampling and sequential Monte Carlo~\cite{IM1989,Neal2001,del2006sequential} framework and ensemble type methods~\cite{Reich2011,EKS,Iglesias_2013,ding2019ensemble,ding2019ensemble2,dingQJ2020ensemble} are also popular.

Different MCMC methods are implemented differently, but they share the essence, that is to develop a Markovian transition kernel whose invariant measure is the target distribution, so that after many rounds of iteration, the invariant measure is achieved. If the design of the transition kernel does not involve $\nabla f$ or sense the local behavior of $f$, the convergence is slow~\cite{MCMCslow1,MCMCslow2,10.1093/biomet/83.1.95,10.2307/2242610}. 

The Langevin Monte Carlo methods, both the overdamped or underdamped, can be viewed as special kinds of MCMC that involve the computation of $\nabla f$. The idea is to find stochastic differential equations (SDEs) whose equilibrium-in-time is the target distribution. These SDEs are typically driven by $\nabla f$, and the Overdamped or Underdamped Langevin Monte Carlo (O/U-LMC) can be viewed as the discrete-in-time (such as Euler-Maruyama discretization) version of the Langevin dynamics (SDEs). Since $\nabla f$ leads the dynamics, fast converge is expected~\cite{DALALYAN20195278,Cheng2017UnderdampedLM,dalalyan2018sampling}.

However, $\nabla f$ is typically not available. In particular, if $f$ is obtained from inverse problems with an underlying governing differential equation describing the dynamics, as seen in the remote sensing and epidemiology examples above, the explicit formula for $\nabla f$ is unknown. When this happens, one usually needs to compute all partial derivatives, one by one, either by employing automatic differentiation~\cite{JMLR:v18:17-468}, or by surrogating with the finite-difference approximations $\partial_{i}f \approx [f(x+\eta \textbf{e}^i)-f(x-\eta \textbf{e}^i)]/2\eta$ for every direction $\textbf{e}^i$. This leads to a cost that is roughly $d$ times the number of required iterations. In high dimension, $d\gg 1$, the numerical cost is high. Therefore, how to sample with a small number of finite differencing approximations with a cost relaxed on $d$, becomes rather crucial.

There are methods proposed to achieve gradient-free property, such as Importance Sampling (IS), Ensemble Kalman methods, random walks methods, and various finite difference approximations to surrogate the gradient. However, IS~\cite{IM1989,Doucet2001,SMCBOOK} has high variance of the weight terms and it leads to wasteful sampling; ensemble Kalman methods~\cite{Evensen:2006:DAE:1206873,BR2010,Reich2011,EKS} usually require Gaussianity assumption~\cite{ding2019ensemble,ding2019ensemble2}; random walk methods such that Metropolized random walk (MRW) \cite{10.2307/2242610,roberts1996,10.1093/biomet/83.1.95}, Ball Walk \cite{89553,Dyer1989ARP,Lovsz1993RandomWI} and the Hit-and-run algorithm \cite{Blisle1993HitandRunAF,Kannan1995IsoperimetricPF,Lovsz2006HitandRunFA} cannot guarantee fast convergence~\cite{VS-2005}; and to our best knowledge, modification of LMC with derivatives replaced by its finite difference approximation~\cite{Meeds2015HamiltonianA} or Kernel Hilbert space \cite{GFHMC} are not yet equipped with theoretical non-asymptotic analysis.

\subsection{Contribution}

We work under the O/U-LMC framework, and we look for methods that produce i.i.d. samples with only a small number of gradient computation. To this end, the contribution of the paper is twofolded.

We first examine a natural strategy to reduce the cost by adopting randomized coordinate descent (RCD)~\cite{doi:10.1137/100802001,Ste-2015}, a random directional gradient approximation. This method replaces $d$ finite difference approximations in $d$ directions, by $1$ in a randomly selected direction. Presumably this reduces the cost in each iteration by $d$ folds, and hopefully the total cost. However, in this article we will show that this is not the case in the general setting. We will provide a counter-example: the high variance induced by the random direction selection process brings up the numerical error, and thus more iterations are needed to achieve the preset error tolerance. This in the end leads to no improvement in terms of the computational cost.

We then propose a variance reduction method to improve the application of RCD to LMC. We call the method Randomized Coordinates Averaging Descent Overdamped/Underdamped LMC (or RCAD-O/U-LMC). The methods start with a fully accurate gradient (up to a discretization error) in the first round of iteration, and in the subsequent iterations they only update the gradient evaluation in one randomly selected direction. Since the methods preserve some information about the gradient along the evolution, the variance is reduced. We prove the new methods converge as fast as the classical O/U-LMC~\cite{DALALYAN20195278,Cheng2017UnderdampedLM,dalalyan2018sampling}, meaning the preset error tolerance is achieved in the same number of iterations. But since they require only $1$ directional derivative  per iteration instead of $d$, the overall cost is reduced. We summarize the advantage over the classical O-LMC and U-LMC in Table~\ref{table:result} (assuming computing the full gradient costs $d$ times of one partial derivative). The dependence on the conditioning of $f$ is omitted in the table, but will be discussed in detail in Section \ref{sec:mainresult}.

In some sense, the new methods share some similarity with SAGA~\cite{SAGA-2014}, a modification of SAG (stochastic average gradient)~\cite{SAGA-2013}. These are two methods designed for reducing variance in the stochastic gradient descent (SGD) framework where the cost function $f$ has the form of $\sum_if_i$. Similar approaches are also found in SG-MCMC (stochastic-gradient Markov chain Monte Carlo (SG-MCMC))~\cite{10.5555/2969442.2969566,10.5555/3044805.3045080,Gao2018GlobalCO,10.5555/3045118.3045176,betancourt2017,NIPS2019_8639,pmlr-v80-chatterji18a}. In their cases, variance reduction is introduced in the selection of $\nabla f_i$. In our case, the cost function $f$ is a simple convex function, but the gradient $\nabla f$ can be viewed as $\nabla f = \sum\partial_i f \textbf{e}^i$ and the variance reduction is introduced in the selection of $\partial_i f \textbf{e}^i$.

There are other variance reduction methods, such as SVRG \cite{DBLP:conf/nips/Johnson013} and CV-ULD \cite{CVULD2017,pmlr-v80-chatterji18a}. We leave the discussion to future research.

\begin{table}[ht]
\begin{center}
\begin{tabular}{ |c|c|c|}
\hline
Algorithm&Number of iterations&Number of $\partial f$ evaluations\\
\hline
O-LMC\cite{DALALYAN20195278}&$\widetilde{O}\left(d/\epsilon\right)$&$\widetilde{O}\left(d^2/\epsilon\right)$\\
\hline
U-LMC\cite{Cheng2017UnderdampedLM,dalalyan2018sampling}&$\widetilde{O}\left(d^{1/2}/\epsilon\right)$&$\widetilde{O}\left(d^{3/2}/\epsilon\right)$\\
\hline
RCAD-O-LMC &$\widetilde{O}\left(d^{3/2}/\epsilon\right)$&$\widetilde{O}\left(d^{3/2}/\epsilon\right)$\\
\hline
RCAD-U-LMC &$\widetilde{O}\left(\max\{d^{4/3}/\epsilon^{2/3},d^{1/2}/\epsilon\}\right)$&$\widetilde{O}\left(\max\{d^{4/3}/\epsilon^{2/3},d^{1/2}/\epsilon\}\right)$\\
 \hline
\end{tabular}
\caption{Number of iterations and directional derivative evaluations of $f(x)$ to achieve $\epsilon$-accuracy. $d$ is the dimension. $\widetilde{O}(f)=O(f\log f)$. {If $g=O(f\log f)$, then $g\leq Cf\log (f)$, where $C$ is a constant independent of $f$.} For the overdamped cases, we assume the Lipschitz continuity for the hessian term. Without this assumption, RCAD-O-LMC still outperforms O-LMC, as will be discussed in Section \ref{sec:mainresult}.}
\label{table:result}
\end{center}
\end{table}

\subsection{Organization}
In Section~\ref{sec:ingredients}, we discuss the essential ingredients of our methods: the random coordinate descent (RCD) method, the overdamped and underdamped Langevin dynamics and the associated Monte Carlo methods (O-LMC and U-LMC). In Section~\ref{sec:notation}, we unify the notations and assumptions used in our methods. In Section~\ref{sec:failuerofrandomsearch}, we discuss the vanilla RCD applied to LMC and present a counter-example to show it is not effective if used blindly. In Section~\ref{sec:mainresult}, we introduce our new methods RCAD-O/U-LMC and present the results on convergence and numerical cost. We demonstrate numerical evidence in Section~\ref{sec:Numer}. Proofs are rather technical and are all left to appendices.

\section{Essential ingredients}\label{sec:ingredients}
\subsection{Random coordinate descent (RCD)}
When explicit formula for $\nabla f$ is not available, one needs to compute the partial derivatives for all directions. One straightforward way is to use finite difference: $\partial_i f(x)\approx \frac{f(x+\eta \textbf{e}^i)-f(x-\eta \textbf{e}^i)}{2\eta}$ where $\textbf{e}^i$ is the $i$-th unit direction. Given enough smoothness, the introduced error is $O(\eta^2)$. For approximating the entire $\nabla f$, $d$ such finite differencing evaluations are required, and it is expensive in the high dimensional setting when $d\gg 1$. The cost is similarly bad if one uses automatic differentiation.

Ideally one can take one random direction and computes the derivative in that direction only, and hopefully this random directional derivative reveals some information of the entire gradient $\nabla f$. This approach is used in both RCD~\cite{Ste-2015,RB2011,doi:10.1137/100802001} and SPSA~\cite{SPSAanalyse1,SPSAanalyse2}. Both methods, instead of calculating the full gradient, randomly pick one direction and use the directional derivative as a surrogate of $\nabla f$. More specifically, RCD computes the derivative in one random unit direction $\textbf{e}^{r}$ and approximates:
\begin{equation}\label{eqn:randomfinitedifferenceRD}
\nabla f\approx d\left(\nabla f(x)\cdot \textbf{e}^{r}\right)\textbf{e}^{r}\approx d\frac{f(x+\eta \textbf{e}^{r})-f(x-\eta \textbf{e}^{r})}{2\eta}\textbf{e}^{r}\,,
\end{equation}
where $r$ is randomly drawn from $1,2,\cdots,d$ (see the distribution of drawing in~\cite{RB2011}).
This approximations is consistent in the expectation sense because
\[
\EE_{r}\left(d\left(\nabla f(x)\cdot \textbf{e}^{r}\right)\textbf{e}^{r}\right)=\nabla f(x)\,.
\]
Here $\EE$ is to take expectation.

\subsection{Overdamped Langevin dynamics and O-LMC}
The O-LMC method is derived from the following Langevin dynamics:
\begin{equation}\label{eqn:Langevin}
\rd X_t=-\nabla f(X_t)\rd t+\sqrt{2}\rd \mathcal{B}_t\,.
\end{equation}
The SDE characterizes the trajectory of $X_t$. The forcing term $\nabla f(X_t)$ and the Brownian motion term $\rd \mathcal{B}_t$ compete: the former drives $X_t$ to the minimum of $f$ and the latter provides small oscillations. The initial data $X_0$ is a random variable drawn from a given distribution induced by $q_0(x)$. Denote $q(x,t)$ the probability density function of $X_t$, it is a well-known result that $q(x,t)$ satisfies the following Fokker-Planck equation:
\begin{equation}\label{eqn:FKPKLangevin}
\partial_t q=\nabla\cdot(\nabla fq+\nabla q)\,,\quad\text{with}\quad q(x,0) = q_0\,,
\end{equation}
and furthermore, $q(x,t)$ converges to the target density function $p(x) = e^{-f}$ exponentially fast in time~\cite{Markowich99onthe}.

The overdamped Langevin Monte Carlo (O-LMC), as a sampling method, can be viewed as a discrete-in-time version of the SDE~\eqref{eqn:Langevin}. A standard Euler-Maruyama method applied on the equation gives:
\begin{equation}\label{eqn:update_ujn}
x^{m+1}=x^m-\nabla f(x^m)h+\sqrt{2h}\xi^{m}\,,
\end{equation}
where $\xi^{m}$ is i.i.d. drawn from $\mathcal{N}(0,I_d)$ with $I_d$ being the identity matrix of size $d$. Since~\eqref{eqn:update_ujn} approximates~\eqref{eqn:Langevin}, the density of $x^m$, denoted as $p_m(x)$, converges to $p(x)$ as $m\to\infty$, up to a discretization error. It was proved in~\cite{DALALYAN20195278} that the convergence to $\epsilon$ is achieved within $\widetilde{O}(d/\epsilon)$ iterations if hessian of $f$ is Lipschitz. If hessian of $f$ is not Lipschitz, the number of iterations increases to $\widetilde{O}(d/\epsilon^2)$. In many real applications, the gradient of $f$ is not available and some approximation is used, introducing another layer of numerical error. In~\cite{DALALYAN20195278}, the authors did discuss the effect of such error, but they assumed the error has bounded variance.

\subsection{Underdamped Langevin dynamics and U-LMC}
The underdamped Langevin dynamics \cite{10.5555/3044805.3045080} is characterized by the following SDE:
\begin{equation}\label{eqn:underdampedLangevin}
\left\{\begin{aligned}
&\rd X_t = V_t\rd t\\
&\rd V_t = -2 V_t\rd t-\gamma\nabla f(X_t)\rd t+\sqrt{4 \gamma}\rd \mathcal{B}_t
\end{aligned}\right.\,,
\end{equation}
where $\gamma>0$ is a parameter to be tuned. Denote $q(x,v,t)$ the probability density function of $(X_t,V_t)$, then $q$ satisfies the Fokker-Planck equation
\[
\partial_tq=\nabla\cdot \left(\begin{bmatrix}
-v\\
2v+\gamma\nabla f
\end{bmatrix}q+\begin{bmatrix}
0 & 0\\
0 & 2\gamma
\end{bmatrix}\nabla q\right)\,,
\]
and under mild conditions, it converges to $p_2(x,v)=\exp(-(f(x)+|v|^2/2\gamma))$, making the marginal density function for $x$ the target $p(x)$~\citep{Villani2006,DOLBEAULT2009511}.

The underdamped Langevin Monte Carlo algorithm, U-LMC, can be viewed as a numerical solver to \eqref{eqn:underdampedLangevin}. In each step, we sample
new particles $(x^{m+1},v^{m+1})\sim (Z^{m+1}_x,Z^{m+1}_v)\in\mathbb{R}^{2d}$, where $(Z_x^{m+1},Z^{m+1}_v)\in\mathbb{R}^{2d}$ is a Gaussian random vector determined by $(x^m,v^m)$ with the following expectation and covariance:
\begin{equation}\label{distributionofZ}
\begin{aligned}
&\EE Z^{m+1}_x=x^m+\frac{1}{2}\left(1-e^{-2h}\right)v^m-\frac{\gamma}{2}\left(h-\frac{1}{2}\left(1-e^{-2h}\right)\right)\nabla f(x^m)\,,\\
&\EE Z^{m+1}_v=v^me^{-2h}-\frac{\gamma}{2}\left(1-e^{-2h}\right)\nabla f(x^m)\,,\\
&\Cov\left(Z^{m+1}_x\right)=\gamma\left[h-\frac{3}{4}-\frac{1}{4}e^{-4h}+e^{-2h}\right]\cdot I_d\,,\ \Cov\left(Z^{m+1}_v\right)=\gamma\left[1-e^{-4h}\right]\cdot I_d\,,\\
&\Cov\left(Z^{m+1}_x\,,Z^{m+1}_v)\right)=\frac{\gamma}{2}\left[1+e^{-4h}-2e^{-2h}\right]\cdot I_d\,.
\end{aligned}
\end{equation}

We here used the notation $\EE$ to denote the expectation, and $\Cov(a,b)$ to denote the covariance of $a$ and $b$. If $b=a$, we abbreviate it to $\Cov(a)$. The scheme can be interpreted as sampling from the following dynamics in each time interval:
\begin{equation*}
\left\{\begin{aligned}
&\mathrm{X}_t=x^m+\int^t_0 \mathrm{V}_s\rd s\\
&\mathrm{V}_t=v^me^{-2t}-\frac{\gamma}{2}(1-e^{-2t})\nabla f(x^m)+\sqrt{4 \gamma}e^{-2 t}\int^t_0e^{2 s}\rd \mathcal{B}_s
\end{aligned}\right.\,.
\end{equation*}

U-LMC does demonstrate faster convergence rate~\cite{Cheng2017UnderdampedLM,dalalyan2018sampling} than O-LMC. Without the assumption on the hessian of $f$ being Lipschitz, the number of iteration is $\widetilde{O}(\sqrt{d}/\epsilon)$ to achieve $\epsilon$ accuracy. The faster convergence on the discrete level could be explained by the better discretization solver instead of faster convergence of the underlying SDEs. Indeed, without the Lipschitz continuity on the hessian term, the discretizing of~\eqref{eqn:underdampedLangevin} produces $O(h^2)$ numerical error. In contrast, the discretization error of~\eqref{eqn:update_ujn} is $O(h^{3/2})$. A third-order discretization was discussed for~\eqref{eqn:underdampedLangevin} in~\cite{mou2019highorder}, further enhancing the numerical accuracy. Similar to O-LMC, the method needs to numerically approximate $\nabla f(x^m)$. This induces another layer of error, and also requires $d$ times of evaluation of $\partial f$.

\section{Notations}\label{sec:notation}
\subsection{Assumption}
We make some standard assumptions on $f(x)$:
\begin{assum}\label{assum:Cov}
The function $f$ is $\mu$-strongly convex and has an $L$-Lipschitz gradient:
\begin{itemize}
\item[--] Convex, meaning for any $x,x'\in\mathbb{R}^d$:
\begin{equation}\label{Convexity}
f(x)-f(x')-\nabla f(x')^\top (x-x')\geq (\mu/2)|x-x'|^2\,.
\end{equation}
\item[--] Gradient is Lipschitz, meaning for any $x,x'\in\mathbb{R}^d$:
\begin{equation}\label{GradientLip}
|\nabla f(x)-\nabla f(x')|\leq L|x-x'|\,.
\end{equation}
\end{itemize}
\end{assum}
If $f$ is second-order differentiable, these assumptions together mean $\mu{I}_d\preceq \mathcal{H}(f)\preceq L{I}_d$ where $\mathcal{H}(f)$ is the hessian of $f$. We also define condition number of $f(x)$ as
\begin{equation}\label{eqn:R}
\kappa=L/\mu\geq 1\,.
\end{equation}
We will express our results in terms of $\kappa$ and $\mu$. Furthermore, for some results we assume Lipschitz condition of the hessian too:
\begin{assum}\label{assum:Hessian}
The function $f$ is second-order differentiable and the hessian of $f$ is H-Lipschitz, meaning for any $x,x'\in\mathbb{R}^d$:
\begin{equation}\label{HessisnLip}
\|\mathcal{H}(f)(x)-\mathcal{H}(f)(x')\|_2\leq H|x-x'|\,.
\end{equation}
\end{assum}

\subsection{Wasserstein distance}
The Wasserstein distance is a classical quantity that evaluates the distance between two probability measures:
\[
W_p(\mu,\nu) = \left(\inf_{(X,Y)\in C(\mu,\nu)} \mathbb{E}|X -Y|^p\right)^{1/p}\,,
\]
where $C(\mu,\nu)$ is the set of distribution of $(X,Y)\in\mathbb{R}^{2d}$ whose marginal distributions, for $X$ and $Y$ respectively, are $\mu$ and $\nu$. These distributions are called the couplings of $\mu$ and $\nu$. Here $\mu$ and $\nu$ can be either probability measures themselves or the measures induced by probability density functions $\mu$ and $\nu$. In this paper we mainly study $W_2$.

\section{Direct application of RCD in LMC, a negative result}\label{sec:failuerofrandomsearch}
We study if RCD can be blindly applied to U-LMC for reducing numerical complexity. This is to replace $\nabla f$ in the updating formula~\eqref{eqn:update_ujn} for U-LMC by the random directional derivative surrogates~\eqref{eqn:randomfinitedifferenceRD}. The resulting algorithms are presented as Algorithm \ref{alg:SOU-LMC} in Appendix \ref{sec:alg:SOUlMC}.

RCD was introduced in optimization. In~\cite{RB2011}, the authors show that despite RCD computes only $1$, instead of $d$ directional derivatives in each iteration, the number of iteration needed for achieving $\epsilon$-accuracy is $O(d/\epsilon)$, as compared to $O(1/\epsilon)$ when the full-gradient is used (suppose Lipschitz coefficient in each direction is at the same order with the total Lipschitz constant). The gain on the cost is mostly reflected by the conditioning of the objective function $f$. This means there are counter-examples for which RCD cannot save compared with ordinary gradient descent. We emphasize that there are of course also plenty examples for which RCD significantly outperforms when $f$ is special conditioning structures~\cite{RB2011,doi:10.1137/100802001,Ste-2015}. In this article we would like to investigate the general lower-bound situations.

The story is the same for sampling. There are examples that show directly applying the vanilla RCD to U-LMC fails to outperform the classical U-LMC. One example is the following: We assume
\[
q_0(x,v)=\frac{1}{(4\pi)^{d/2}}\exp(-|x-\textbf{u}|^2/2-|v|^2/2)\,,\quad p_2(x,v)=\frac{1}{(2\pi)^{d/2}}\exp(-|x|^2/2-|v|^2/2)\,,
\]
where $\textbf{u}\in\mathbb{R}^d$ satisfies $\textbf{u}_i=1/8$ for all $i$. Denote $\{(x^m,v^m)\}$ the sample computed through Algorithm \ref{alg:SOU-LMC} (underdamped) with stepsize $h$. Let $\eta$ be extremely small and the finite differencing error is negligible, and denote $q_m$ the probability density function of $(x^m,v^m)$, then we can show $W_2(q_m,p_2)$ cannot converge too fast.
\begin{theorem}\label{thm:badexampleW22}
For the example above, choose $\gamma=1$, there exists uniform nonzero constant $C_1$ such that if $d,h$ satisfy
\[
d>2,\quad h<\left\{\frac{1}{100(1+C_1)},\frac{1}{1440^2d}\right\}\,,
\]
then
\begin{equation}\label{eqn:badexampleW2bound2}
W_m\geq \exp\left(-2mh\right)\frac{\sqrt{d}}{1024}+\frac{d^{3/2}h}{2304}\,,
\end{equation}
where $W_m=W_2(q^U_{m},p_2)$, and $q^U_m(x,v)$ is the probability density function of $m$-th iteration of RCD-U-LMC.
\end{theorem}

  The proof is found in Section \ref{sec:badexample}. We note the second term in \eqref{eqn:badexampleW2bound2} is rather big. The smallness comes from $h$, the stepsize, and it needs be small enough to balance out the influence from $d^{3/2}\gg 1$. This puts strong restriction on $h$. Indeed, to have $\epsilon$-accuracy, $W(q_m,p_2)\leq\epsilon$, we need both terms smaller than $\epsilon$, and this term suggests that $h\leq\frac{2304\epsilon}{d^{3/2}}$ at least. And when combined with restriction from the first term, we arrive at the conclusion that at least $\widetilde{O}(d^{3/2}/\epsilon)$ iterations are needed, and thus $\widetilde{O}(d^{3/2}/\epsilon)$ finite differencing approximation are required. The $d$ dependence is $d^{3/2}$, and is exactly the same as that in U-LMC, meaning RCD-U-LMC brings no computational advantage over U-LMC in terms of the dependence on the dimension of the problem.

We emphasize that that large second term, as shown in the proof, especially in Section \ref{sec:badexample} equation \eqref{highvariance}, is induced exactly due to the high variance in the gradient approximation. This triggers our investigation into variance reduction techniques.

\section{Random direction approximation with variance reduction on O/U-LMC, two positive results}\label{sec:mainresult}
The direct application of RCD induces high variance and thus high error. It leads to many more rounds of iterations for convergence, gaining no numerical saving in the end. In this section we propose RCAD-O/U-LMC with RCAD reducing variance in the framework of RCD. We will prove that while the numerical cost per iteration is reduced by $d$-folds, the number of required iteration is mostly unchanged, and thus the total cost is reduced.
\subsection{Algorithm}
The key idea is to compute one accurate gradient at the very beginning in iteration No. $1$, and to preserve this information along the iteration to prevent possible high variance. The algorithms for RCAD-O-LMC and RCAD-U-LMC are both presented in Algorithm \ref{alg:SAGA-OU-LMC}, based on overdamped and underdamped Langevin dynamics respectively. Potentially the same strategy can be combined with SPSA, which we leave to future investigation.

In the methods, an accurate gradient (up to a finite-differencing error) is used in the first step, denoted by $g\approx \nabla f$, and in the subsequent iterations, only one directional derivative of $f$ gets computed and updated in $g$.

\begin{algorithm}[htb]
\caption{\textbf{Randomized Coordinate Averaging Decent O/U-LMC (RCAD-O/U-LMC)}}\label{alg:SAGA-OU-LMC}
\begin{algorithmic}
\State \textbf{Preparation:}
\State 1. Input: $\eta$ (space stepsize); $h$ (time stepsize); $\gamma$ (parameter); $d$ (dimension); {$M$ (stopping index)} and $f(x)$.
\State 2. Initial: \emph{(overdamped)}: $x^0$ i.i.d. sampled from a initial distribution induced by $q_0(x)$ and calculate $g^0\in\mathbb{R}^d$:
\begin{equation}\label{eqn:g1}
g^0_i=\frac{f(x^0+\eta \textbf{e}^i)-f(x^0-\eta \textbf{e}^i)}{2\eta},\quad \ 1\leq i\leq d\,.
\end{equation}
\emph{(underdamped)}: $(x^0,v^0)$ i.i.d. sampled from a initial distribution induced by $q_0(x,v)$ and calculate $g^0\in\mathbb{R}^d$ as in~\eqref{eqn:g1}.
\State \textbf{Run: }\textbf{For} $m=0\,,1\,,\cdots {M}$

1. Draw a random number $r^m$ uniformly from $1,2,\cdots,d$.

2. Calculate $g^{m+1}$ and flux $F^m\in \mathbb{R}^d$ by letting $g^{m+1}_{i}=g^{m}_{i}$ for $i\neq r_m$ and 
\begin{equation}\label{randomgradientapproximation}
g^{m+1}_{r_m}=\frac{f(x^m+\eta \textbf{e}^{r_m})-f(x^m-\eta \textbf{e}^{r_m})}{2\eta}\,,\quad F^m=g^m+d\left(g^{m+1}-g^m\right)\,.
\end{equation}

3. \emph{(overdamped)}: Draw $\xi^{m}$ from $\mathcal{N}(0,I_d)$:
\begin{equation}\label{eqn:update_ujnSASGOLD}
x^{m+1}=x^m-F^m h+\sqrt{2h}\xi^{m}\,.
\end{equation}
\emph{(underdamped)}: Sample $(x^{m+1},v^{m+1})\sim Z^{m+1}=(Z^{m+1}_x,Z^{m+1}_v)$ where $Z^{m+1}$ is a Gaussian random variable with expectation and covariance defined in~\eqref{distributionofZ}, replacing $\nabla f(x^m)$ by $F^m$.

\State \textbf{end}
\State \textbf{Output:} $\{x^m\}$.
\end{algorithmic}
\end{algorithm}

\subsection{Convergence and numerical cost analysis}
We now discuss the convergence of RCAD-O-LMC and RCAD-U-LMC, and compare the results with the classical O-LMC and U-LMC methods~\cite{DALALYAN20195278,Cheng2017UnderdampedLM}. We emphasize that these two papers indeed discuss the numerical error in approximating the gradients, but they both require the variance of error being bounded, which is not the case here. One related work is \cite{pmlr-v80-chatterji18a}, where the authors construct the Lyapunov function to study the convergence of SG-MCMC. Our proof for the convergence of RCAD-O-LMC is inspired by its technicalities. In~\cite{Cheng2017UnderdampedLM,pmlr-v80-chatterji18a}, a contraction map is used for U-LMC, but such map cannot be directly applied to our situation because the variance depends on the entire trajectory of samples. Furthermore, the history of the trajectory is reflected in each iteration, deeming the process to be non-Markovian. We need to re-engineer the iteration formula accordingly for tracing the error propagation.

\subsubsection{Convergence for RCAD-O-LMC}
For RCAD-O-LMC, we have the following theorem:
\begin{theorem}\label{thm:disconvergenceSAGAOLD}
Suppose $f$ satisfies Assumption \ref{assum:Cov}-\ref{assum:Hessian} and $h,\eta$ satisfy
\begin{equation}\label{eqn:conditiononheta}
h<\frac{1}{3(1+9d)\kappa ^2\mu},\quad \eta<h\,.
\end{equation}
Then $W_2(q^O_m\,,p)$, the Wasserstein distance between $q^O_m$, the probability density function of the sample $x^m$ derived from Algorithm \ref{alg:SAGA-OU-LMC} (overdamped), and $p$, the target density function, satisfies
\begin{equation}\label{iterationinequalityofSAGA}
W_2(q^O_m,p)\leq \exp(-\mu hm/4)\sqrt{1+1/\kappa ^2}W_2(q^O_0,p)+2h\sqrt{d^3C_1+d^2C_2}\,.
\end{equation}
Here $C_1=77\kappa ^2\mu $, $C_2=H^2/\mu ^2+20\kappa ^2+\kappa ^3\mu /d$.
\end{theorem}
See proof in Appendix~\ref{proofofRCAD-O-LMC}. The theorem gives us the strategy of designing stopping criterion: to achieve $\epsilon$-accuracy, meaning to have $W_2(q^O_m,p)\leq \epsilon$, we can choose to set both terms in~\eqref{iterationinequalityofSAGA} less than $\epsilon/2$, and it leads to:
\[
h\leq \min\left\{\frac{1}{3(1+9d)\kappa ^2\mu },\frac{\epsilon}{4d^{3/2}\sqrt{C_1+C_2/d}}\right\}
\]
and
\[
M\geq \frac{4}{h\mu }\log\left(\frac{2\sqrt{1+1/\kappa ^2}W_2(q_0,p)}{\epsilon}\right)\,.
\]
This means the cost, also the number of $\partial f$ evaluations, is $\widetilde{O}(d^{3/2}/\epsilon)$.

Note that the theorem here requires both Assumptions~\ref{assum:Cov} and~\ref{assum:Hessian}. We can relax the second assumption. If so, the numerical cost of degrades to $\widetilde{O}(\max\{d^{3/2}/\epsilon,d/\epsilon^2\})$, whereas the cost of O-LMC is $\widetilde{O}(d^2/\epsilon^2)$. Our strategy still outperforms. The proof is the same, and we omit it from the paper.

\subsubsection{Convergence for RCAD-U-LMC}
For RCAD-U-LMC, we have the following theorem.

\begin{theorem}\label{thm:thmconvergenceULDSAGA} Assume $f(x)$ satisfies Assumption \ref{assum:Cov}, and set $\gamma= 1/L$, then there exists a uniformly constant $D>0$ such that if $h,\eta$ satisfy 
\begin{equation}\label{eqn:conditionuetaULDSAGA}
h\leq \min\left\{\frac{1}{100(1+D)\kappa },\frac{1}{1648\kappa d}\right\},\quad \eta<h^3\,,
\end{equation}
then $W_2(q^U_m\,,p_2)$, the Wasserstein distance between the distribution of the sample $(x^m,v^m)$, derived from Algorithm \ref{alg:SAGA-OU-LMC} (underdamped), and distribution induced by $p_2$ (whose marginal density in $x$ is $p$) decays as:
\begin{equation}\label{eqn:convergeULDSAGA}
\begin{aligned}
W_2(q^U_m,p_2)\leq &4\sqrt{2}\exp(-hm/(8\kappa ))W_2(q^U_0,p_2)\\
&+600\sqrt{h^3d^4/\mu }+200\sqrt{\kappa h^2d/\mu }+350\sqrt{\kappa h^5d^2}
\end{aligned}\,.
\end{equation}
\end{theorem}
See proof in Appendix ~\ref{proofofRCAD-U-LMC}. To achieve $\epsilon$-accuracy, meaning to have $W_2(q^U_m,p_2)\leq \epsilon$, we can choose all terms in~\eqref{eqn:convergeULDSAGA} less than $\epsilon/4$. This gives:
\[
h\leq \min\left\{\frac{\epsilon^{2/3}\mu ^{1/3}}{(2400)^{2/3}d^{4/3}},\frac{\epsilon \mu ^{1/2}}{800\kappa ^{1/2}d^{1/2}},\frac{\epsilon^{2/5}}{(1400)^{2/5}\kappa ^{1/5}d^{2/5}},\frac{1}{(1+D)\kappa },\frac{1}{1648\kappa d}\right\}
\]
and thus the stopping index needs to be:
\[
M\geq \frac{8\kappa }{h}\log\left(\frac{16\sqrt{2}W_2(q^U_0,p_2)}{\epsilon}\right)\,.
\]
This means $\widetilde{O}\left(\max\left\{d^{4/3}/\epsilon^{2/3},d^{1/2}/\epsilon\right\}\right)$ evaluations of $\partial f$.
\section{Numerical result}\label{sec:Numer}
We demonstrate numerical evidence in this section. We first note that it is extremely difficult to compute the Wasserstein distance between two probability measures in high dimensional problems, especially when they are represented by a number of samples. The numerical result below evaluates a weaker measure:
\begin{equation}\label{errortest}
\mathrm{Error}=\left|\frac{1}{N}\sum^{N}_{i=1} \phi(x^{M,i})-\EE_p(\phi)\right|\,,
\end{equation}
where $\phi$ is the test function. $\{x^{M,i}\}_{i=1}^N$ are $N$ different samples iterate till $M$-th step, and $p$ is the target distribution.

In the first example, our target distribution is $\mathcal{N}(0,I_d)$ with $d=1000$, and in the second example we use
\[
p(x)\propto \exp\left(-\sum^d_{i=1}\frac{|x_i-2|^2}{2}\right)+\exp\left(-\sum^d_{i=1}\frac{|x_i+2|^2}{2}\right)\,.
\]
For both example, we sample the initial particles according to $\mathcal{N}(0.5,I_d)$. We run both RCD-O/U-LMC and RCAD-O/U-LMC using $N=5\times 10^5$ particles and test MSE error with $\phi(x)=|x_1|^2$ in both examples. In Figure~\ref{Figure1} and Figure~\ref{Figure2} respectively we show the error with respect to different stepsizes. In all the computation, $M$ is big enough. The improvement of adding variance reduction technique is obvious in both examples.
\begin{figure}[htbp]
     \centering
     \subfloat{\includegraphics[height = 0.15\textheight, width = 0.46\textwidth]{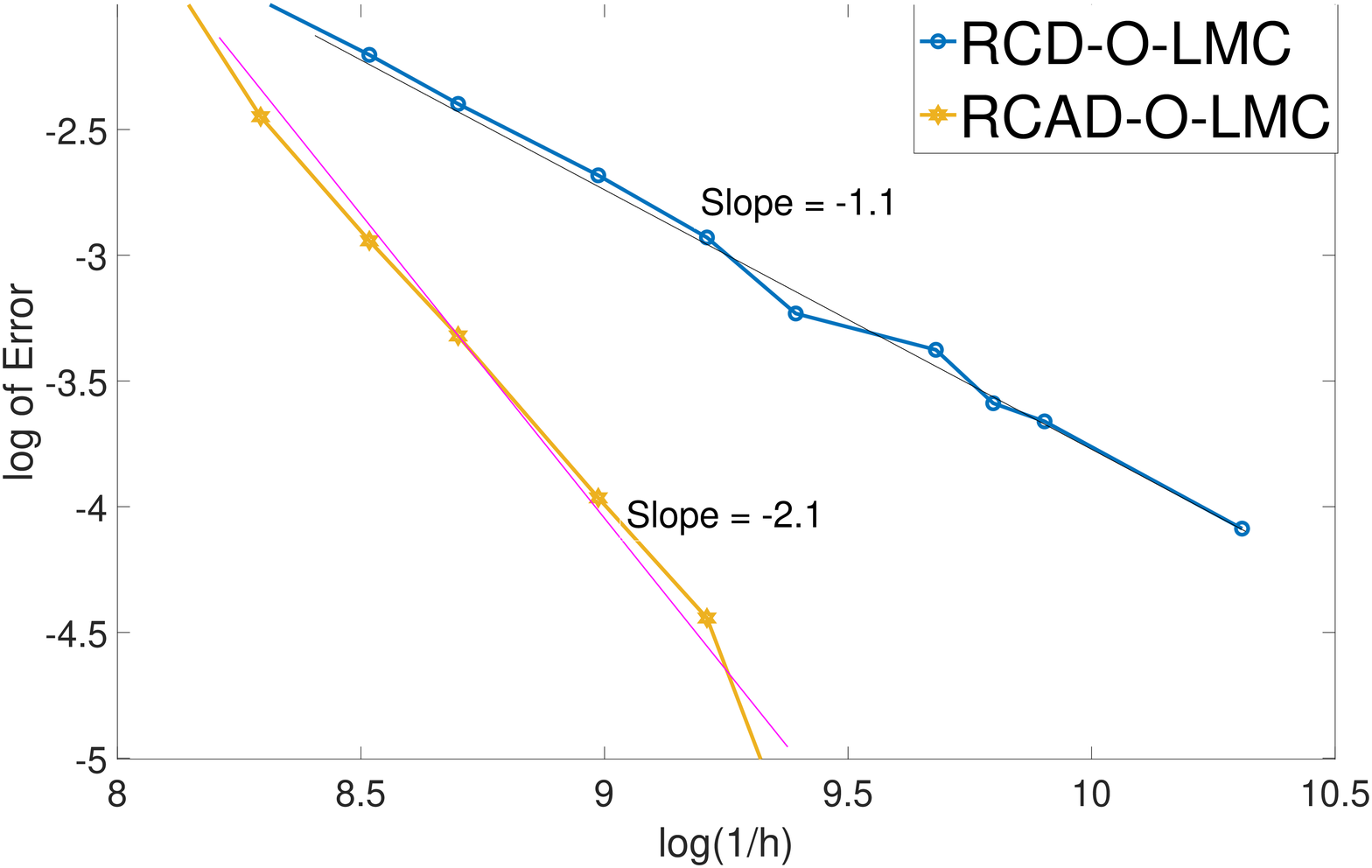}}\hspace{0.04\textwidth}
     \subfloat{\includegraphics[height = 0.15\textheight, width = 0.46\textwidth]{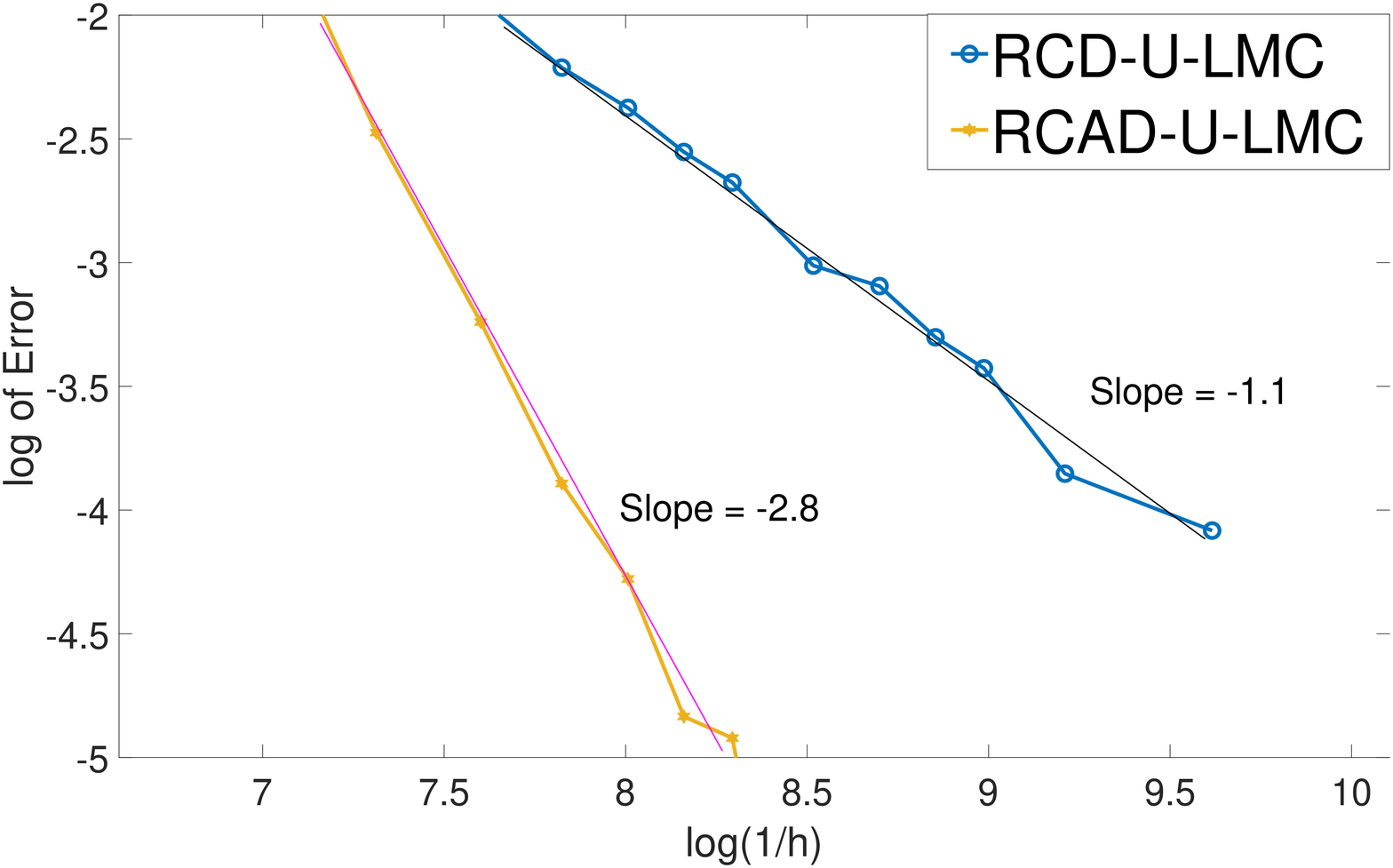}}
     \caption{Example 1. Decay of Error of O-LMC (left) and U-LMC (right) with and without RCAD.}
     \label{Figure1}
\end{figure}
\begin{figure}[htbp]
      \centering
      \subfloat{\includegraphics[height = 0.15\textheight, width = 0.46\textwidth]{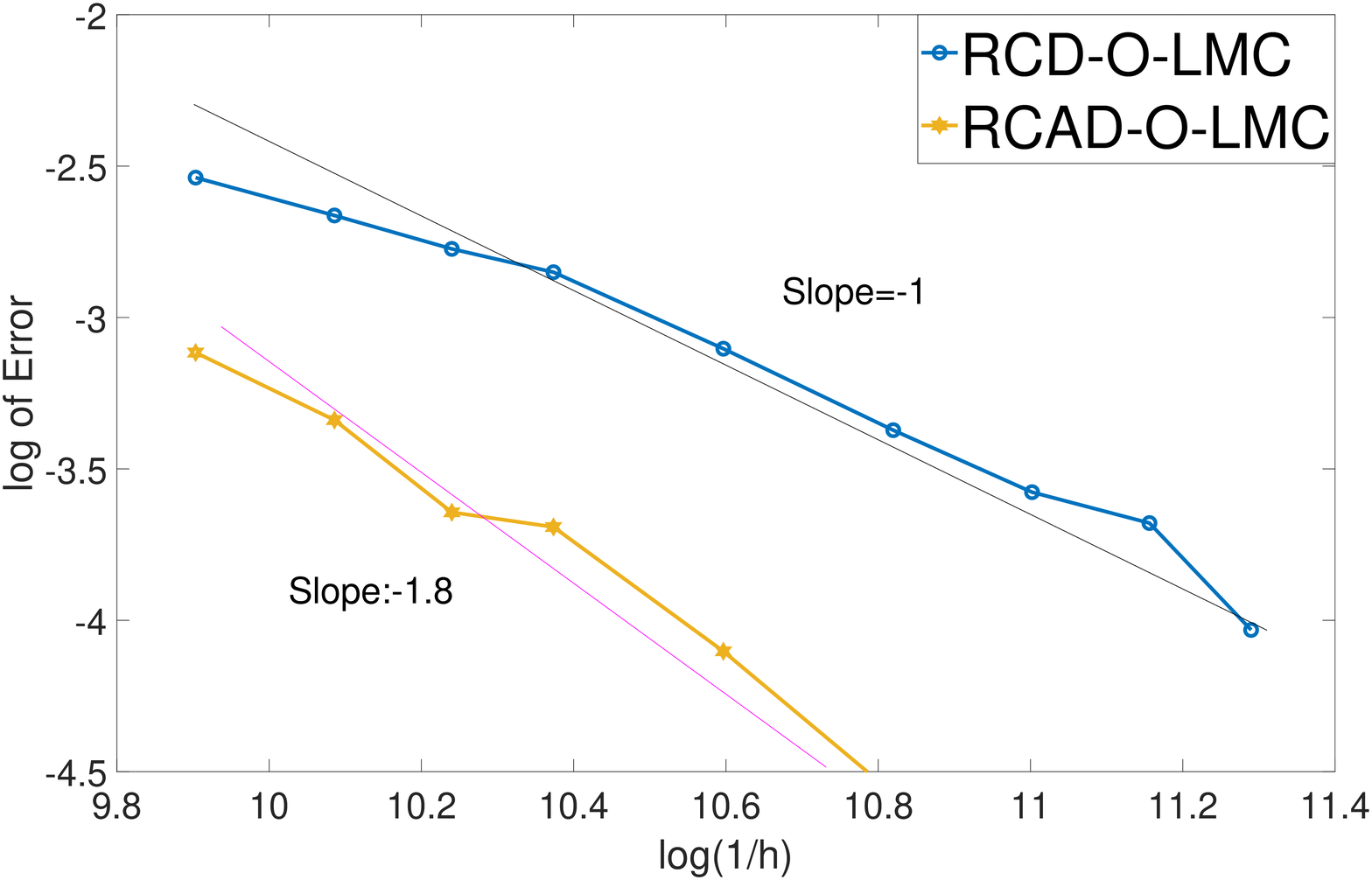}}\hspace{0.04\textwidth}
      \subfloat{\includegraphics[height = 0.15\textheight, width = 0.46\textwidth]{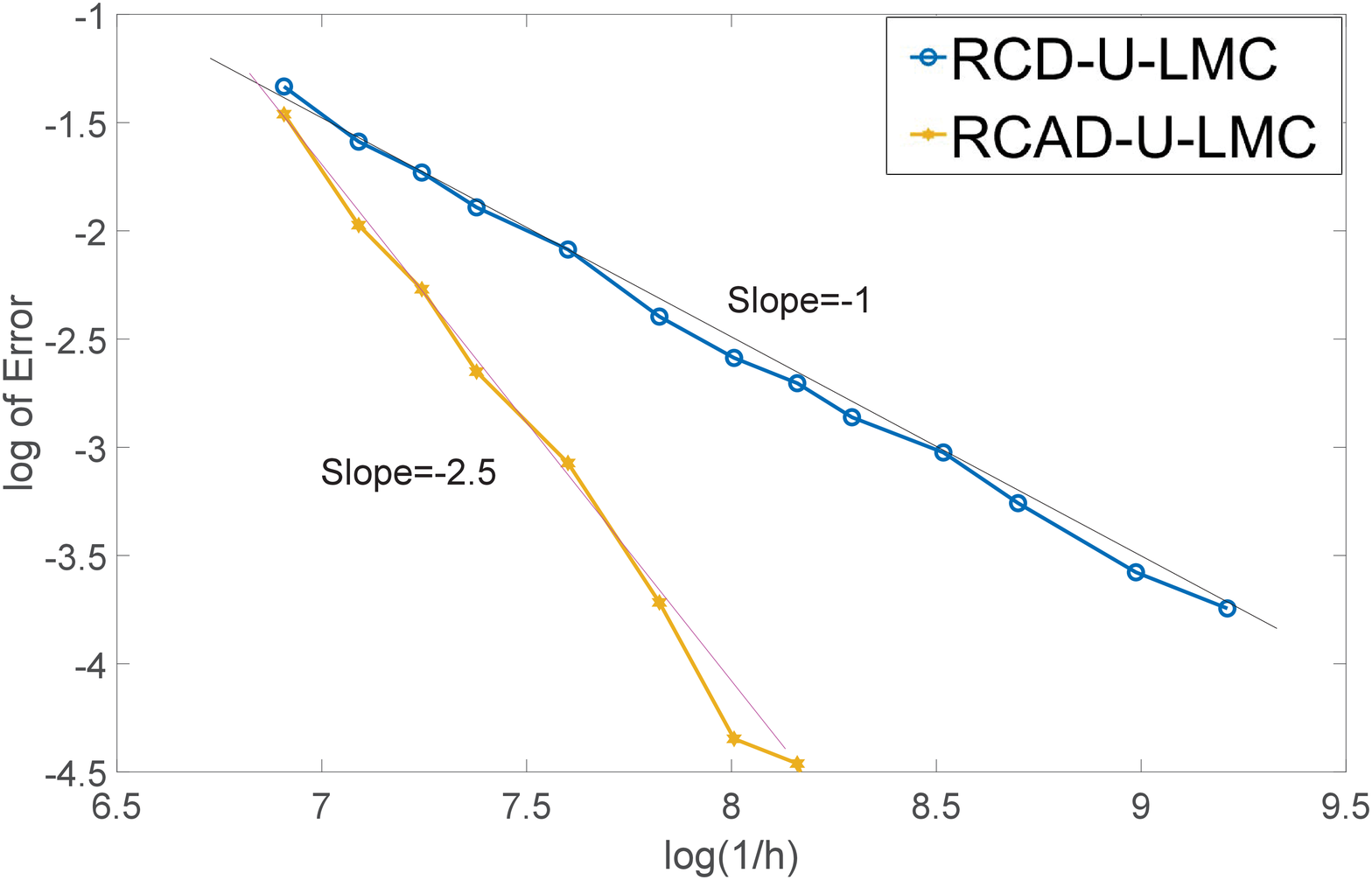}}
      \caption{Example 2. Decay of Error of O-LMC (left) and U-LMC (right) with and without RCAD.}
      \label{Figure2}
 \end{figure}

\section{Conclusion and future work}
To our best knowledge, this is the first work that discusses both the negative and positive aspects of applying random gradient approximation, mainly RCD type, to LMC, in both overdamped and underdamped situations without and with variance reduction. Without variance reduction we show the RCD-LMC has the same numerical cost as the classical LMC, and with variance reduction, the numerical cost is reduced in both overdamped and underdamped cases.

There are a few future directions that we would like to pursue. 1. Our method, in its current version, is blind to the structure of $f$. The only assumptions are reflected on the Lipschitz bounds. In~\cite{RB2011,doi:10.1137/100802001,doi:10.1137/130949993} the authors, in studying optimization problems, propose to choose random directions according to the Lipschitz constant in each direction. The idea could potentially be incorporated in our framework to enhance the sampling strategy. 2. Our algorithms are designed based on reducing variance in the RCD framework. Potentially one can also apply variance reduction methods to improve SPSA-LMC. There are also other variance reduction methods that one could explore.

\section{Broader Impact}

The result provides theoretical guarantee to the application of random coordinate descent to Langevin Monte Carlo, when variance reduction technique is used to reduce the cost. It has potential application to inverse problems emerging from atmospheric science, remote sensing, and epidemiology. This work does not present any foreseeable societal consequence.

\begin{ack}
Both authors acknowledge generous support from NSF-DMS 1750488, NSF-TRIPODS 1740707, Wisconsin Data Science Initiative, and Wisconsin Alumni Research Foundation.
\end{ack}
\bibliographystyle{plain}
\bibliography{enkf}

\begin{appendix}
\section{Algorithms and Results of RCD-LMC}\label{sec:SPSAmainresult}
\subsection{Algorithm}\label{sec:alg:SOUlMC}
We apply RCD as surrogates of the gradient in O/U-LMC. This amounts to replacing the gradient terms in \eqref{eqn:update_ujn} using the approximation~\eqref{eqn:randomfinitedifferenceRD}. The new methods are presented in Algorithm~\ref{alg:SOU-LMC}, termed RCD-O/U-LMC.
\begin{algorithm}[htb]
\caption{\textbf{RCD-overdamped(underdamped) Langevin Monte Carlo}}\label{alg:SOU-LMC}
\begin{algorithmic}
\State \textbf{Preparation:}
\State 1. Input: $\eta$ (space step); $h$ (time step); $\gamma$ (parameter); $d$ (dimension); $M$ (stopping index) and $f(x)$.
\State 2. Initial: \emph{(overdamped)}: $x^0$ i.i.d. sampled from a initial distribution induced by $q_0(x)$.

\emph{(underdamped)}: $(x^0,v^0)$ i.i.d. sampled from the initial distribution induced by $q_0(x,v)$.

\State \textbf{Run: }\textbf{For} $m=0\,,1\,,\cdots M$

1. Finite difference: calculate flux approximation by RCD:
\begin{equation}\label{eqn:RCD}
F^m=d\frac{f(x^m+\eta \textbf{e}^{r})-f(x^m-\eta \textbf{e}^{r})}{2\eta}\textbf{e}^{r}
\end{equation}
with $r$ uniformly drawn from $1\,,\cdots,d$.

2. \emph{(overdamped)}: Draw $\xi^{m}$ from $\mathcal{N}(0,I_d)$:
\begin{equation}\label{eqn:update_ujnSD}
x^{m+1}=x^m-F^m h+\sqrt{2h}\xi^{m}\,.
\end{equation}
\emph{(underdamped)}: Sample $(x^{m+1},v^{m+1})\sim Z^{m+1}=(Z^{m+1}_x,Z^{m+1}_v)$ where $Z^{m+1}$ is a Gaussian random variable with expectation and covariance defined in~\eqref{distributionofZ}, replacing $\nabla f(x^m)$ by $F^m$.
\State \textbf{end}
\State \textbf{Output:} $\{x^m\}$.
\end{algorithmic}
\end{algorithm}

\subsection{A counter-example}\label{sec:badexample}
In this section, we prove Theorem \ref{thm:badexampleW22}.

Fisrt, we define $w^m=x^m+v^m$, and denote $u_m(x,w)$ the probability density of $(x^m,w^m)$ and $u^\ast(x,w)$ the probability density of $(x,w)$ if $(x,v=w-x)$ is distributed according to density function $p_2$. From~\cite{Cheng2017UnderdampedLM}, we have:
\begin{equation}\label{trivialinequlaitySAGAcounterexample}
 |x^m-x|^2+|v^m-v|^2\leq 4(|x^m-x|^2+|w^m-w|^2)\leq 16(|x^m-x|^2+|v^m-v|^2)\,,
 \end{equation}
 and thus
 \begin{equation}\label{trivialinequlaitySAGA2counterexample}
  W^2_2(q^U_{m},p_2)\leq  4W^2_2(u_{m},u^*)\leq  16W^2_2(q^U_{m},p_2)\,.
\end{equation}

\begin{proof}[Proof of Theorem \ref{thm:badexampleW22}] 
Throughout the proof, we drop the superscript ``$U$" to have a concise notation. According to \eqref{trivialinequlaitySAGA2counterexample}, it suffices to find a lower bound for $W^2_2(u_{m},u^*)$. We first notice
\begin{equation}\label{iterationW2badbound}
\begin{aligned}
W_2(u_{m},u^{\ast})\geq&\sqrt{\int |w|^2u_m(x,w)\rd w\rd x}-\sqrt{\int|w|^2u^{\ast}(x,w)\rd w\rd x}\\
=&\sqrt{\int |w|^2u_m(x,w)\rd w\rd x}-\sqrt{2d}=\sqrt{\EE|w^m|^2}-\sqrt{2d}\\
=&\frac{\EE|w^m|^2 -2d}{\sqrt{\EE|w^m|^2}+\sqrt{2d}}\,,
\end{aligned}
\end{equation}
where $\EE$ takes all randomness into account. This implies to prove \eqref{eqn:badexampleW2bound2}, it suffices to find a lower bound for second moment of $w^m$. Indeed, in the end, we will show that
\begin{equation}\label{ddd}
W_2(u_m\,,u^\ast)\geq \left(1-2h\right)^m\frac{\sqrt{d}}{512}+\frac{d^{3/2}h}{1152}
\end{equation}
and thus
\[
W_2(q_m,p_2)\geq \left(1-2h\right)^m\frac{\sqrt{d}}{1024}+\frac{d^{3/2}h}{2304}
\]
proving the statement of the theorem since $\left(1-2h\right)^m\geq \exp\left(-2mh\right)$. To show~\eqref{ddd}, we first note, by direct calculation:
\begin{equation}\label{bbb}
W_2(q_0,p_2)=\sqrt{d}/8\,,\quad \EE|\omega^0|^2=\frac{129d}{64}\,,
\end{equation}
then we divide the proof into several steps:
\begin{itemize}
\item \textbf{First step:} \emph{Priori estimation}

According to \eqref{bbb}, use convergence result of Algorithm \ref{alg:SOU-LMC} (\cite{ZQLMC} Theorem 4.1), we have for any $m\geq 0$
\[
W_2(q_m,p_2)\leq \frac{\sqrt{d}}{2}+\frac{\sqrt{d}}{6}=\frac{2\sqrt{d}}{3}\,,
\]
Similar to \eqref{iterationW2badbound}, we have
\[
W_2(q_m,p_2)\geq |\sqrt{\EE|x^m|^2}-\sqrt{d}|\,,\quad W_2(q_m,p)\geq |\sqrt{\EE|v^m|^2}-\sqrt{d}|
\]
which implies
\begin{equation}\label{lowerboundforx}
\frac{\sqrt{d}}{3}\leq \sqrt{\EE|x^m|^2}\leq \frac{5\sqrt{d}}{3},\quad \frac{\sqrt{d}}{3}\leq \sqrt{\EE|v^m|^2}\leq \frac{5\sqrt{d}}{3}\,
\end{equation}
for any $m\geq 0$. 

Finally, use \eqref{lowerboundforx}, we can obtain
\begin{equation}\label{upperboundforw}
\sqrt{\EE|\omega^m|^2}\leq \sqrt{\EE|x^m|^2}+\sqrt{\EE|v^m|^2}\leq 4\sqrt{d}\,.
\end{equation}

\item \textbf{Second step:} \emph{Iteration formula of $\EE|w^m|^2$.}

By the special structure of $p$, we can calculate the second moment explicitly. Since $f(x)$ can be written as
\[
f(x)=\sum^d_{i=1}\frac{|x_i|^2}{2}\,,
\]
in each step of RCD-U-LMC, according to Algorithm \ref{alg:SOU-LMC}, for each $m\geq 0$ and $1\leq i\leq d$, we have
\begin{equation}\label{distributionofZ1}
\begin{aligned}
&\EE \left(x^{m+1}_i|(x^m,v^m,r^m)\right)=x^m_i+\frac{1-e^{-2h}}{2}v^m_i-\left(\frac{h}{2}-\frac{1-e^{-2h}}{4}\right)(x^m_i-E^m_i),\\
&\EE \left(v^{m+1}_i|(x^m,v^m,r^m)\right)=v^m_ie^{-2h}-\frac{1-e^{-2h}}{2}(x^m_i-E^m_i)\,,\\
&\EE \left(w^{m+1}_i|(x^m,v^m,r^m)\right)=\frac{1+e^{-2h}}{2}w^m_i+\frac{1-e^{-2h}}{2}E^m_i-\left(\frac{h}{2}-\frac{1-e^{-2h}}{4}\right)(x^m_i-E^m_i)\,,\\
&\Var\left(x^{m+1}_i|(x^m,v^m,r^m)\right)=h-\frac{3}{4}-\frac{1}{4}e^{-4h}+e^{-2h}\,,\\
&\Var\left(v^{m+1}_i|(x^m,v^m,r^m)\right)=1-e^{-4h}\,,\\
&\Cov\left((x^{m+1}_i\,,v^{m+1}_i)|(x^m,v^m,r^m)\right)=\frac{1}{2}\left[1+e^{-4h}-2e^{-2h}\right]\,,
\end{aligned}
\end{equation}
where $E^m\in\mathbb{R}^d$ is a random variable defined as
\[
E^m_i=x^m_i-dx^m_i\textbf{e}^{r_m}_i
\]
and satisfies
\begin{equation}\label{highvariance}
\EE_{r_m}(E^m_i)=0,\quad \EE_{r_m}\left|E^m_i\right|^2=(d-1)|x^m_i|^2
\end{equation}
for each $1\leq i\leq d$. Furthermore,
\begin{equation}\label{aaaa}
\EE\langle w^m_i\,,E^m_i\rangle = \EE\langle x^m_i\,,E^m_i\rangle = 0\,.
\end{equation}

Now, since $h\leq \frac{1}{880}$, we can replace $e^{-2h}$ and $e^{-4h}$ by their Taylor expansion:
\begin{equation}\label{taylorofe}
e^{-2h}=1-2h+2h^2+D_1h^3,\quad e^{-4h}=1-4h+8h^2+D_2h^3\,,
\end{equation}
where $D_1,D_2$ are negative constants depends on $h$ and satisfy
\[
|D_1|<10,\quad |D_2|<100\,.
\]
Plug \eqref{taylorofe} into \eqref{distributionofZ1}, we have
\begin{equation}\label{distributionofZ2}
\begin{aligned}
&\EE \left(w^{m+1}_i|(x^m,w^m,r^m)\right)=\left(1-h+h^2+\frac{D_1h^3}{2}\right)w^m_i-\left(\frac{h^2}{2}+\frac{D_1h^3}{4}\right)x^m_i\\
&\quad\quad\quad\quad\quad\quad\quad\quad\quad\quad\quad\quad+\left(h-\frac{h^2}{2}-\frac{D_1h^3}{4}\right)E^m_i\,,\\
&\mathrm{Var}\left(x^{m+1}_i|(x^m,v^m,r^m)\right)=\left(D_1-\frac{D_2}{4}\right)h^3\,,\\
&\Var\left(v^{m+1}_i|(x^m,v^m,r^m)\right)=4h-8h^2-D_2h^3\,,\\
&\Cov\left((x^{m+1}_i\,,v^{m+1}_i)|(x^m,v^m,r^m)\right)=2h^2+\frac{\left(D_2-2D_1\right)h^3}{2}\,.
\end{aligned}
\end{equation}
The last three equalities in \eqref{distributionofZ2} implies
\[
\begin{aligned}
\Var\left(w^{m+1}_i|(x^m,v^m,r^m)\right)=&\Var\left(x^{m+1}_i|(x^m,v^m,r^m)\right)+\Var\left(v^{m+1}_i|(x^m,v^m,r^m)\right)\\
&+2\Cov\left(x^{m+1}_i\,,v^{m+1}_i)|(x^m,v^m,r^m)\right)\\
=&4h-4h^2-\left(D_1+\frac{D_2}{4}\right)h^3\,.
\end{aligned}
\]
Then, we can calculate the iteration formula for $\EE|x^{m+1}_i|^2$ and $\EE|\omega^{m+1}_i|^2$:
\[
\begin{aligned}
&\EE|\omega^{m+1}_i|^2\\
=&\EE_{x^m,w^m,r^m}\left(\left.|\omega^{m+1}_i|^2\right|(x^m,v^m,r^m)\right)\\
=&\EE_{x^m,w^m,r^m}\left(\left|\EE \left(w^{m+1}_i|(x^m,w^m,r^m)\right)\right|^2+\Var\left(w^{m+1}_i|(x^m,v^m,r^m)\right)\right)\\
=&\left(1-h+h^2+\frac{D_1h^3}{2}\right)^2\EE|w^m_i|^2+(d-1)\left(h-h^2-\frac{D_1h^3}{2}\right)^2\EE|x^m_i|^2\\
&+\left(\frac{h^2}{2}+\frac{D_1h^3}{4}\right)^2\EE|x^m_i|^2-2\left(1-h+h^2+\frac{D_1h^3}{2}\right)\left(\frac{h^2}{2}+\frac{D_1h^3}{4}\right)\EE\left\langle w^m_i,x^m_i\right\rangle\\
&+4h-4h^2-\left(D_1+\frac{D_2}{4}\right)h^3\,,
\end{aligned}
\]
where we use \eqref{highvariance} and \eqref{aaaa}\,. 

Sum them up with $i$, we finally obtain an iteration formula for $\EE|w^m|^2$:
\begin{equation}\label{omegamiexpectationomega}
\begin{aligned}
&\EE|\omega^{m+1}|^2\\
=&\left(1-h+h^2+\frac{D_1h^3}{2}\right)^2\EE|w^m|^2+
(d-1)\left(h-h^2-\frac{D_1h^3}{2}\right)^2\EE|x^m|^2\\
&+\left(\frac{h^2}{2}+\frac{D_1h^3}{4}\right)^2\EE|x^m|^2-2\left(1-h+h^2+\frac{D_1h^3}{2}\right)\left(\frac{h^2}{2}+\frac{D_1h^3}{4}\right)\EE\left\langle w^m,x^m\right\rangle\\
&+4dh-4dh^2-d\left(D_1+\frac{D_2}{4}\right)h^3\,.
\end{aligned}
\end{equation}

\item \textbf{Third step:} \emph{Lower bound for $W_2(u_{m},u^*)$}
Use \eqref{upperboundforw}, since $D_1<0$, $h<\frac{1}{100}<\frac{1}{|D_1|}$ and $d>1872$, we have
\[
\frac{h^2}{2}\leq h^2+\frac{D_1h^3}{2}\leq h^2,\quad h-h^2-\frac{D_1h^3}{2}\geq h-h^2\geq \frac{h}{2}\,,
\]
which implies
\begin{equation}\label{inequalfirst}
1-h+h^2+\frac{D_1h^3}{2}\geq 1-h+h^2/2\,,\quad (d-1)\left(h-h^2-\frac{D_1h^3}{2}\right)^2\EE|x^m|^2\geq \frac{d^2h^2}{72}\,.
\end{equation}
and
\begin{equation}\label{inequalsecond}
\left(\frac{h^2}{2}+\frac{D_1h^3}{4}\right)\EE\left\langle w^m,x^m\right\rangle\leq \frac{h^2}{2}\left(\EE|w^m|^2\EE|x^m|^2\right)^{1/2}\leq 4dh^2\,.
\end{equation}

For the last line of \eqref{omegamiexpectationomega}, since $h\leq \frac{1}{880}<\frac{1}{|D_1|+|D_2|/4}$, we have
\begin{equation}\label{inequalthird}
4dh-4dh^2-d\left(D_1+\frac{D_2}{4}\right)h^3\geq 4dh-5dh^2\,.
\end{equation}
Plug \eqref{inequalfirst},\eqref{inequalsecond},\eqref{inequalthird} into \eqref{omegamiexpectationomega}, we have
\begin{equation}\label{lowerboundofomegam0}
\EE|\omega^{m+1}|^2\geq\left(1-h+h^2/2\right)^2\EE|w^m|^2+4dh+d^2h^2/72-13dh^2\,.
\end{equation}
Note that $\left(1-h+h^2/2\right)^2\geq 1-2h$, and since $d>1872$, we have $\frac{d^2h^2}{144}\geq 13dh^2$. Use~\eqref{lowerboundofomegam0} iteratively and combine with~\eqref{bbb}, we finally have:
\begin{equation}\label{lowerboundofomegam}
\begin{aligned}
\EE|\omega^{m}|^2&\geq\frac{129\left(1-2h\right)^md}{64}+\left(1-(1-2h)^m\right)\left[2d+d^2h/288\right]\\
&=(1-2h)^m\left[\frac{d}{64}-\frac{d^2h}{288}\right]+2d+\frac{d^2h}{288}\\
&\geq (1-2h)^m\frac{d}{128}+2d+\frac{d^2h}{288}\,,
\end{aligned}
\end{equation}

Plug \eqref{lowerboundofomegam} into \eqref{iterationW2badbound}, we further have
\[
\begin{aligned}
W_2(u_{m},u^*)&\geq \frac{(1-2h)^m\frac{d}{128}+2d+\frac{d^2h}{288}-2d}{\sqrt{(1-2h)^m\frac{d}{128}+2d+\frac{d^2h}{288}}+\sqrt{2d}}\\
&\geq \frac{(1-2h)^m\frac{d}{128}+\frac{d^2h}{288}}{4\sqrt{d}}\\
&\geq \left(1-2h\right)^m\frac{\sqrt{d}}{512}+\frac{d^{3/2}h}{1152}\,,
\end{aligned}
\]
where we use small enough $h$ in the second inequality to bound the $d^2h$ term by $d$ in the denominator. This proves \eqref{ddd}.
\end{itemize}
\end{proof}
\section{Proof of convergence of RCAD-O-LMC (Theorem \ref{thm:disconvergenceSAGAOLD})}\label{proofofRCAD-O-LMC}

In this section we provide the detailed proof for Theorem~\ref{thm:disconvergenceSAGAOLD}.

Before diving into details, we quickly summarize the proving strategy. Recall that the target distribution $p$ is merely the equilibrium of the SDE~\eqref{eqn:Langevin}. This means, if a particle prepared at the initial stage is drawn from $p$, then following the dynamics of SDE~\eqref{eqn:Langevin}, the distribution of this particle will continue to be $p$. In the analysis below, we call the trajectory of this particle $y_t$, and the sequence generated by this particle evaluated at discrete time $y^m$. Essentially we evaluate how quickly $x^m$ converges to $y^m$ as $m$ increases. In particular, we call $\Delta^m = x^m-y^m$ and will derive an iteration formula that shows the convergence of $\Delta^m$.

In evaluating $\Delta^{m}$, there are three kinds of error that get involved:
\begin{itemize}
\item[1.] discretization error in $\eta$: this can be made as small as possible. $\eta$ is a spatial stepsize parameter and can be made as small as we wish. The finite differencing accuracy is second order and thus the produced error is at the order of $\mathcal{O}(\eta^2)$. By making $\eta$ small, we make this part of error negligible;
\item[2.] discretization error in $h$: this amounts to controlling the discretization error of the SDE~\eqref{eqn:Langevin}. To handle this part of error we employ the estimates in~\cite{DALALYAN20195278};
\item[3.] random coordinate selection process error: this is to measure, at each iteration, how big can $\nabla f(x^m)-F^m$ be. According to the way $F^m$ is defined, it is straightforward to show that the expectation of this error is always $0$, but the variance $\EE|\nabla f(x^m)-F^m|^2$ can be big, and this is the main reason for the direct application of RCD on LMC to fail~\cite{ZQLMC}. The variance reduction technique discussed in this paper is exactly to reduce the size of this term.
\end{itemize}

In a way, see details in~\eqref{eqn:Deltam+1}, we can derive the iteration formula, ignoring the discretization error in $\eta$, 
\begin{align*}
\Delta^{m+1}
=&y^{m+1} - x^{m+1}=\Delta^m+(y^{m+1}-y^{m})-(x^{m+1}-x^m)\\
=&\Delta^m-h\left(\nabla f(y^m)-\nabla f(x^m)\right)-\int^{(m+1)h}_{mh}\left(\nabla f(y_s)-\nabla f(y^m)\right)\rd s\\
&-h(\nabla f(x^m)-F^m)\,.
\end{align*}

The second term on the right hand side, by using the Lipschitz continuity, will provide $-Lh\Delta^m$, and it produces desirable property when combined with the first $\Delta^m$. The third term encodes the discretization error in $h$, and was shown to be small in~\cite{DALALYAN20195278}. The last term is the error that comes from the random coordinate selection process. We discuss it in details in Section~\ref{VarwidetileEX}. We note that this term cannot be simply controlled and estimated by $\Delta^m$ only, but $\partial_{r^m} f(x^m) - \partial_{r^m} f(y^m)$ as well. If we simply relax it to $L\Delta^m$, we will lose the decay property brought by the second term. To overcome that, we define the Lyapunov function that combines the effects of $\Delta^m$ and $\partial_{r^m} f(x^m) - \partial_{r^m} f(y^m)$. See definition in~\eqref{eqn:Lyna}.

Now we prove the theorem in details. After a lengthy definition of all notations, we will present Lemma~\ref{lemma:SAGAOLD} and Lemma~\ref{lemmaTM2+1}. They are to bound, iteratively $\Delta^m$ and $\partial_{r^m} f(x^m) - \partial_{r^m} f(y^m)$ term respectively. The proof of the theorem then follows by combining the two lemmas to control the Lyapunov function.

As presented in the main text, the first step of RCAD-O-LMC uses the finite differencing approximation for every direction, namely, setting $g^0\in\mathbb{R}^d$ to be:
\[
g^0_i=\frac{f(x^0+\eta \textbf{e}^{i})-f(x^0-\eta \textbf{e}^{i})}{2\eta},\quad i=1,2,\cdots,d\,.
\]
In the following iterations, one random direction is selected for the updating,
\[
g^{m+1}_{r_m}=\frac{f(x^m+\eta \textbf{e}^{r_m})-f(x^m-\eta \textbf{e}^{r_m})}{2\eta}
\]
with other directions untouched: $g^{m+1}_i = g^m_i$ for all $i\neq r_m$. Define:
\[
F^m=g^m+d\left(g^{m+1}-g^m\right)\,,
\]
then the updating formula is:
\begin{equation}\label{eqn:update_ujnSAGA}
x^{m+1}=x^m-F^mh+\sqrt{2h}\xi^{m}\,,
\end{equation}
where $h$ is the time stepsize, and $\xi^{m}$ i.i.d. drawn from $\mathcal{N}(0,I_{d})$. Denote
\begin{equation}\label{eqn:errorofapproximate}
E^m=\nabla f(x^m)-F^m\,,
\end{equation}
then this updating formula~\eqref{eqn:update_ujnSAGA} writes to:
\begin{equation}\label{eqn:update_ujnSAGA2}
x^{m+1}=x^m-\nabla f(x^m)h+E^mh+\sqrt{2h}\xi^{m}\,.
\end{equation}
This is the formula we use for the analysis under Assumptions~\ref{assum:Cov} and~\ref{assum:Hessian}.

To show the theorem, we let $y_0$ be a random vector drawn from target distribution induced by $p$ such that $W^2_2(q^O_0,p)=\EE|x^0-y^0|^2$, and set
\begin{equation}\label{eqn:yt}
y_t=y_0-\int^t_0 \nabla f(y_s)\rd s+\sqrt{2}\int^t_0\rd \mathcal{B}_s\,,
\end{equation}
where we construct the Brownian motion that always satisfies
\begin{equation}\label{eqn:Bt}
B_{h(m+1)}-B_{hm}=\sqrt{h}\xi^m\,.
\end{equation}
Then $y_t$ is drawn from target distribution as well. On the discrete level, let $y^m=y_{mh}$, then:
\[
y^{m+1}=y^{m}-\int^{(m+1)h}_{mh} \nabla f(y_s)\rd s+\sqrt{2h}\xi^{m}\,.
\]

Noting
\[
W^2_2(q^O_m,p)\leq\EE|x^m-y^m|^2\,,
\]
where $\EE$ takes all randomness into account. We now essentially need to show the difference between~\eqref{eqn:update_ujnSAGA} and~\eqref{eqn:yt}, also see~\cite{pmlr-v80-chatterji18a}. 

As for a preparation, we now define an a set of auxiliary gradients.
\begin{itemize}
\item $\widetilde{g}^0$ is the true derivative used at the initial step:
\begin{equation}\label{initialg0}
\widetilde{g}^0=\nabla f(x^0)\,,
\end{equation}
\item $\widetilde{g}^{m+1}$ is the continuous version of $g^{m+1}$:
\begin{equation}\label{updateg}
\widetilde{g}^{m+1}_{r_m}=\partial_{r_m}f(x^m)\quad\text{and}\quad \widetilde{g}^{m+1}_{i}=\widetilde{g}^{m}_{i}\quad \text{if}\quad i\neq r_m\,,
\end{equation}
\item $\widetilde{F}^m$ is the continuous version of $F^m$:
\[
\widetilde{F}^m=\widetilde{g}^m+d\left(\widetilde{g}^{m+1}-\widetilde{g}^m\right)\,.
\]

\item Define $\beta^m$ using \eqref{initialg0},\eqref{updateg} with the same $r_m$ but replacing $x^m$ with $y^m$:
\[
\beta^0=\nabla f(y^0)
\]
and
\[
\beta^{m+1}_{r_m}=\partial_{r_m}f(y^m)\quad\text{and}\quad \beta^{m+1}_{i}=\beta^{m}_{i}\quad \text{if}\quad i\neq r_m\,.
\]
\end{itemize}
Indeed in the later proof we will give an upper bound for the following Lyapunov function:
\begin{equation}\label{eqn:Lyna}
T^m=T^m_1+c_pT^m_2=\EE|y^m-x^m|^2+c_{p}\EE|\widetilde{g}^m-\beta^m|^2\,.
\end{equation}
where $c_{p}$ will be carefully chosen later.

We further define
\[
\widetilde{E}^m=\nabla f(x^m)-\widetilde{F}^m=E^m+F^m-\widetilde{F}^m\,,
\]
this leads to $E^m = \widetilde{E}^m - F^m+\widetilde{F}^m$. The properties of $\widetilde{E}^m$ will be discussed in Appendix \ref{VarwidetileEX}. To quantify $F^m-\widetilde{F}^m$ is straightforward: it can be bounded using mean-value theorem. Since:
\[
|\widetilde{g}^0_i-g^0_i|^2=\left|\frac{f(x^m+\eta \textbf{e}^{i})-f(x^m-\eta \textbf{e}^{i})-2\eta\partial_i f(x^m)}{2\eta}\right|^2\leq\left|\frac{(\partial_{i}f(z)-\partial_i f(x^m))2\eta}{2\eta}\right|^2\leq L^2\eta^2\,
\]
where $z\in\mathbb{R}^d$ is a point between $x^m+\eta \textbf{e}^{i}$ and we use the fact that $\nabla f$ is $L$-Lipschitz. Similarly, for all $m$:
\[
|\widetilde{g}^m-g^m|^2\leq L^2\eta^2d\,,
\]
we have:
\begin{equation}\label{errorinG}
\begin{aligned}
\left|\widetilde{F}^m-F^m\right|^2\leq &2|\widetilde{g}^m-g^m|^2+2d^2|\widetilde{g}^{m+1}
_{r_m}-g^{m+1}_{r_m}|^2<2L^2\eta^2d+8L^2\eta^2d^2\,.
\end{aligned}\,.
\end{equation}

Now we present the iteration formula for $T^{m+1}_1,T^{m+1}_2$, in Lemma~\ref{lemma:SAGAOLD} and Lemma~\ref{lemmaTM2+1} respectively:
\begin{lemma}\label{lemma:SAGAOLD}
Under conditions of Theorem \ref{thm:disconvergenceSAGAOLD}, for any $a>0$, we can upper bound  $T^m_1$:
\begin{equation}\label{firstimboundOLD}
T^{m+1}_1\leq (1+a)AT^m_1 + (1+a)BT^m_2 + (1+a)h^3C +\left(1+\frac{1}{a}\right)h^4D
\end{equation}
where
\begin{equation*}
\begin{aligned}
&A=1-2\mu h+3(1+3d)L^2h^2\,,\quad B = 9h^2d\,,\\
&C=2L^2d+72L^2d^3\left[\frac{hL^2d}{\mu }+1\right]\,,\quad D=(H^2+16L^2)d^2+(L^3+4L^2)d\,.
\end{aligned}
\end{equation*}
\end{lemma}
Note that for the proof to proceed, one at least needs the coefficient $(1+a)A<1$. This can be made possible only if $a$ is small enough. For small $a$, the $h^4D$ term is magnified, but it may not matter as $h^4$ serves as a high order error so the term is negligible so long as $a\gg h^4$.
\begin{proof}
Define $\Delta^m=y^m-x^m$, we first divide $\Delta^{m+1}$ into several parts:
\begin{equation}\label{eqn:Deltam+1}
\begin{aligned}
\Delta^{m+1}=&\Delta^m+(y^{m+1}-y^{m})-(x^{m+1}-x^m)\\
=&\Delta^m+\left(-\int^{(m+1)h}_{mh}\nabla f(y_s)\rd s+\sqrt{2h}\xi_m\right)-\left(-\int^{(m+1)h}_{mh}F^m\rd s+\sqrt{2h}\xi_m\right)\\
=&\Delta^m-\left(\int^{(m+1)h}_{mh}\left(\nabla f(y_s)-F^m\right)\rd s\right)\\
=&\Delta^m-\left(\int^{(m+1)h}_{mh}\left(\nabla f(y_s)-\nabla f(y^m)+\nabla f(y^m)-\nabla f(x^m)+\nabla f(x^m)-F^m\right)\rd s\right)\\
=&\Delta^m-h\left(\nabla f(y^m)-\nabla f(x^m)\right)-\int^{(m+1)h}_{mh}\left(\nabla f(y_s)-\nabla f(y^m)\right)\rd s\\
&-h(\nabla f(x^m)-F^m)\\
=&\Delta^m-hU^m-(V^m+h\Phi^m)-hE^m\\
=&\Delta^m-(V^m+h(\widetilde{F}^m-F^m))-h(U^m+\Phi^m+\widetilde{E}^m)
\end{aligned}\,,
\end{equation}
where we set\footnote{In particular, it is obvious that the square of all terms except $\Delta^m$ contribute small values and will enter $d$, and the cross terms would dominate.}
\[
\begin{aligned}
U^m&=\nabla f(y^m)-\nabla f(x^m)\,,\\
V^m&=\int^{(m+1)h}_{mh}\left(\nabla f(y_s)-\nabla f(y^m)-\sqrt{2}\int^s_{mh}\mathcal{H}(f)(y_r)\rd B_r\right)\rd s\,,\\
\Phi^m&=\frac{\sqrt{2}}{h}\int^{(m+1)h}_{mh}\int^s_{mh}\mathcal{H}(f)(y_r)\rd B_r\rd s\,.
\end{aligned}
\]

Upon getting equation~\eqref{eqn:Deltam+1} it is time to analyze each term and hopefully derive an induction inequality that states $\EE|\Delta^{m+1}|^2\approx c\EE|\Delta^m|^2+d$ with $c<1$ and $d$ being of high order in $\eta$ and $h$, some parameters we can tune. Indeed the $\Delta^m$ term is what we would like to preserve, and the $U^m$ term depends on $\Delta^m$ with a Lipschitz coefficient. The opposite signs of these two terms essentially indicate that $c$ can be made $<1$. The $V^m+h\Phi^m$ completely depends on the one-time step error. In some sense, it is close to the forward Euler error obtained in one timestep. The $\widetilde{E}^m$ term is the most crucial term and the only term that reflects the error introduced by the algorithm in one time step. By choosing the right discretization in the algorithm to approximate $\nabla f$, one could expect this term to be small. We leave the analysis of this term to Appendix~\ref{VarwidetileEX}, and focus on how the other terms interact here. 

We first control last two terms in the last line of~\eqref{eqn:Deltam+1}. According to Lemma 6 of~\cite{DALALYAN20195278}, we first have
\begin{equation}\label{boundforVmphim}
\EE|V^m|^2\leq \frac{h^4}{2}\left(H^2d^2+L^3d\right),\quad \EE|\Phi^m|^2\leq \frac{2L^2hd}{3}\,,
\end{equation}

and thus:
\begin{equation}\label{secondDeltam+1}
\begin{aligned}
\EE|V^m+h(\widetilde{F}^m-F^m)|^2&\leq 2\left(\EE|V^m|^2+h^2\EE|\widetilde{F}^m-F^m)|^2\right)\\
&\leq h^4\left(H^2d^2+L^3d\right)+2h^2\left(2L^2\eta^2d+8L^2\eta^2d^2\right)\\
&\leq (H^2+16L^2)h^4d^2+(L^3+4L^2)h^4d = h^4D\,,
\end{aligned}
\end{equation}
where we use~\eqref{errorinG} and \eqref{boundforVmphim} in the second inequality and the condition of $h$ and $\eta$ in~\eqref{eqn:conditiononheta} in last inequality. We also have:
\begin{equation}\label{thirdDeltam+1}
\begin{aligned}
&\EE|U^m+\Phi^m+\widetilde{E}^m|^2\\
\leq& 3\EE|U^m|^2+3\EE|\Phi^m|^2+3\EE|\widetilde{E}^m|^2\,,\\
\leq& 3L^2T^m_1+2L^2hd+9dL^2T^m_1+9dT_2^m+72hL^2d^3\left[\frac{hL^2d}{\mu }+1\right]\,.
\end{aligned}
\end{equation}
where we used the Lipschitz continuity of $f$ for controlling $U^m$, \eqref{boundforVmphim} for $\Phi^m$, and Appendix \ref{VarwidetileEX} for $\widetilde{E}^m$. 

We then handle the cross terms. For example, due to the independence,~\eqref{expectationofEx}, and the convexity, we have:
\begin{equation}\label{eqn:cross_UE}
\EE\left\langle\Delta^m,\Phi^m\right\rangle=0\,,\quad \EE\left\langle\Delta^m,\widetilde{E}^m\right\rangle=0\,,\quad \left\langle\Delta^m,U^m\right\rangle \geq \mu |\Delta^m|^2\,,
\end{equation}
this means the cross term between first and the third term in the last line~\eqref{eqn:Deltam+1} leads to $-2\mu h\EE|\Delta^m|^2$. The cross term produced by the first and the last term, however can be hard to control, mostly because $\EE\left\langle\Delta^m,V^m\right\rangle$ is unknown. We now employ Young's inequality, meaning, for any $a>0$:
\begin{equation}\label{Deltam+1}
\begin{aligned}
T_1^{m+1}&=\EE|\Delta^{m+1}|^2\\
&\leq (1+a)\EE|\Delta^{m+1}+V^m+h(\widetilde{F}^m-F^m)|^2+\left(1+\frac{1}{a}\right)\EE|V^m+h(\widetilde{F}^m-F^m)|^2\,.
\end{aligned}
\end{equation}
While the second term is already investigated in~\eqref{secondDeltam+1}, the first term of~\eqref{Deltam+1}, according to~\eqref{eqn:Deltam+1} becomes:
\begin{equation}\label{firsttermofV}
\begin{aligned}
\EE|\Delta^{m+1}+V^m+h(\widetilde{F}^m-F^m)|^2=&\EE|\Delta^m-h(U^m+\Phi^m+\widetilde{E}^m)|^2\\
=&\EE|\Delta^{m}|^2-2h\EE\left\langle\Delta^m,U^m+\Phi^m+\widetilde{E}^m \right\rangle\\
&+h^2\EE|U^m+\Phi^m+\widetilde{E}^m|^2\\
\leq &(1-2\mu h)\EE|\Delta^{m}|^2+h^2\EE|U^m+\Phi^m+\widetilde{E}^m|^2
\end{aligned}\,,
\end{equation}
where we used~\eqref{eqn:cross_UE}. Plug\eqref{thirdDeltam+1} into \eqref{firsttermofV}, we have have, using the definition of the coefficients $A, B, C$:
\begin{equation}\label{firsttermofV2}
\begin{aligned}
\EE|\Delta^m-h(U^m+\Phi^m+\widetilde{E}^m)|^2\leq AT^m_1+Ch^3+BT^m_2\,,
\end{aligned}
\end{equation}
and plug it together with~\eqref{secondDeltam+1} in~\eqref{Deltam+1} to conclude~\eqref{firstimboundOLD}.
\end{proof}

\begin{lemma}\label{lemmaTM2+1}
Under conditions of Theorem~\ref{thm:disconvergenceSAGAOLD}, we have the upper bound for $T^{m+1}_2$:
\begin{equation}\label{secondimboundOLD}
T^{m+1}_2\leq \tilde{A}T^m_1+\tilde{B}T^m_2
\end{equation}
where $\tilde{A}=\frac{L^2}{d}$ and $\tilde{B} = 1-1/d$.
\end{lemma}
Note that the coefficient $\tilde{B}$ is automatically $<1$ and the gap $1/d$ is independent of $h$ and $\eta$. This gives us some room to tune the parameters.
\begin{proof}
We now expand $\EE\left|\beta^{m+1}_i-\widetilde g^{m+1}_i\right|^2$:
\[
\begin{aligned}
\EE_{r_m}\left|\beta^{m+1}_i-\widetilde g^{m+1}_i\right|^2&=\EE_{r_m}\left[\left|\beta^{m+1}_i-\widetilde g^{m+1}_i\right|^2-\left|\beta^m_i-\widetilde g^{m}_i\right|^2\right]+\left|\beta^{m}_i-\widetilde g^{m}_i\right|^2\\
&=\frac{1}{d}\left[\left|\partial_i f(y^m)-\partial_i f(x^m)\right|^2-\left|\beta^m_i-\widetilde g^{m}_i\right|^2\right]+\left|\beta^{m}_i-\widetilde g^{m}_i\right|^2\\
&=\left(1-\frac{1}{d}\right)\left|\beta^m_i-\widetilde g^{m}_i\right|^2+\frac{1}{d}\left|\partial_i f(y^m)-\partial_i f(x^m)\right|^2
\end{aligned}\,.
\] 
Therefore, we have
\begin{equation}\label{iterationforhg}
\begin{aligned}
\EE\left|\beta^{m+1}-\widetilde g^{m+1}\right|^2&= \left(1-\frac{1}{d}\right)\EE\sum^d_{i=1}\left|\beta^m_i-\widetilde g^{m}_i\right|^2+\frac{1}{d}\EE|\nabla f(y^m)-\nabla f(x^m)|^2\\
&\leq \left(1-\frac{1}{d}\right)\EE\left|\beta^m-\widetilde g^{m}\right|^2+\frac{L^2}{d}\EE|\Delta^m|^2
\end{aligned}
\end{equation}
\end{proof}

Now, we are ready to prove Theorem \ref{thm:disconvergenceSAGAOLD} by adjusting $a$ and $c_p$.

\begin{proof}[Proof of Theorem \ref{thm:disconvergenceSAGAOLD}]
Plug \eqref{firstimboundOLD} and~\eqref{secondimboundOLD} into \eqref{eqn:Lyna}, we have
\begin{equation}\label{firstiterationformula}
\begin{aligned}
T^{m+1}\leq &\left((1+a)A+c_p\tilde{A}\right)T^m_1+\left(\frac{(1+a)B}{c_p}+\tilde{B}\right)c_pT^m_2\\
&+(1+a)h^3C+\left(1+\frac{1}{a}\right)h^4D\,.
\end{aligned}
\end{equation}
To show the proof amounts to choosing proper $c_p$ and $a$. Note that according to the definitions, $A\sim 1-\mu h$, $\tilde{A}\sim 1/d$, $B\sim h^2$ and $\tilde{B}\sim 1-1/d$, this suggests $c_p\sim h^2$ to cancel out the order in $B$, and in the end we have estimates of the form:
\[
(1+a)A + c_p\tilde{A}=1-O(h)\,,\quad \frac{(1+a)B}{c_p}+\tilde{B}=1-O(h)\,.
\]
Indeed, let us choose
\[
c_p=18(1+a)h^2d^2\,,
\]
so that
\[
(1+a)A + c_p\tilde{A} = (1+a)(1-2\mu h+3(1+9d)L^2h^2)\,,\quad\text{and}\quad \frac{(1+a)B}{c_p}+\tilde{B} = 1-\frac{1}{2d}\,.
\]
Since $h$ satisfies~\eqref{eqn:conditiononheta}, this relaxes them to
\[
(1+a)A + c_p\tilde{A} \leq (1+a)(1-\mu h)\,,\quad\text{and}\quad \frac{(1+a)B}{c_p}+\tilde{B} = 1-\frac{1}{2d}\leq 1-\frac{\mu h}{2}\,.
\]
Setting $a=\frac{\mu h/2}{1-\mu h}$ so that
\[
(1+a)(1-\mu h)= 1-\frac{\mu h}{2}\,,\quad\text{and}\quad 1+1/a\leq 2/\mu h\,,
\]
and this finally leads to
\begin{equation}\label{firstiterationformula2}
\begin{aligned}
T^{m+1}&\leq (1-\mu h/2)T^m_1+(1-\mu h/2)c_pT^m_2+2h^3C+\frac{2}{\mu h}h^4D\\
&\leq (1-\mu h/2)T^{m}+2\left(h^3C+h^3D/\mu \right)\,.
\end{aligned}
\end{equation}
Noting
\[
W^2_2(q^O_m,p)\leq T^m
\]
and
\[
\begin{aligned}
T^{0}&=\EE|y^0-x^0|^2+c_{p}\EE|g^0-\beta^0|^2=\EE|y^0-x^0|^2+c_{p}\EE|\nabla f(x^0)-\nabla f(y^0)|^2\\
&\leq (1+c_pL^2)\EE|y^0-x^0|^2\leq (1+\mu ^2/L^2)W^2_2(q^O_0,p)\leq (1+1/\kappa ^2)W^2_2(q^O_0,p)\,,
\end{aligned}
\]
where we use $c_pL^2\leq 36h^2L^2d^2$ and $hLd<\mu /(27L)$, by iteration, we finally have
\begin{equation}\label{eqn:secondlast}
W^2_2(q^O_m,p)\leq \exp(-\mu hm/2)(1+1/\kappa ^2)W^2_2(q_0^O,p)+4\left(h^2C/\mu +h^2D/\mu ^2\right)\,.
\end{equation}
The proof is concluded considering
\[
C/\mu \leq d^3\left(2L^2/(d^2\mu )+75L^2/\mu \right)\leq 77d^3\kappa ^2\mu \,,
\]
\[
D/\mu ^2\leq d^2(H^2/\mu ^2+20\kappa ^2+\kappa ^3\mu /d)\,.
\]
\end{proof}

\section{Proof of convergence of RCAD-U-LMC (Theorem \ref{thm:thmconvergenceULDSAGA})}\label{proofofRCAD-U-LMC}
Recall the definitions:
\begin{itemize}
\item $E^m$: $E^m = \nabla f(x^m) - F^m$
\item $\widetilde{g}^0: \widetilde{g}^0=\nabla f(x^0)$
\item $\widetilde{g}^{m+1}$: $\widetilde{g}^{m+1}_{r_m}=\partial_{r_m}f(x^m)\quad\text{and}\quad \widetilde{g}^{m+1}_{i}=\widetilde{g}^{m}_{i}\quad \text{if}\quad i\neq r_m\,,$
\item $\widetilde{F}^m$: $\widetilde{F}^m=\widetilde{g}^m+d\left(\widetilde{g}^{m+1}-\widetilde{g}^m\right)$
\item $\widetilde{E}^m$: $\widetilde{E}^m=\nabla f(x^m)-\widetilde{F}^m=E^m+F^m-\widetilde{F}^m$
\end{itemize}
Similarly, we also have
\begin{equation}\label{errorinG1}
|\widetilde{g}^m-g^m|^2\leq L^2\eta^2d,\quad \left|\widetilde{F}^m-F^m\right|^2=\left|\widetilde{E}^m-E^m\right|^2\leq 2L^2\eta^2d+8L^2\eta^2d^2\,.
\end{equation}

According to the algorithm, RCAD-U-LMC can be seen as drawing $(x^0,v^0)$ from distribution induced by $q^U_0$, and
update $(x^m,v^m)$ using the following coupled SDEs:
\begin{equation}\label{eqn:ULDSDE2SAGA}
\left\{\begin{aligned}
&\mathrm{V}_t=v^me^{-2(t-mh)}-\gamma \int^{t}_{mh}e^{-2(t-s)}\rd sF^m+\sqrt{4\gamma}e^{-2 (t-mh)}\int^{t}_{mh}e^{2 s}d\mathcal{B}_s \\
&\mathrm{X}_t=x^m+\int^{t}_{mh} \mathrm{V}_sds
\end{aligned}\right.\,,
\end{equation}
where $\mathcal{B}_s$ is the Brownian motion and $(x^{m+1},v^{m+1})=(\mathrm{X}_{(m+1)h},\mathrm{V}_{(m+1)h})$. 

We then define $w^m=x^m+v^m$, and denote $u_m(x,w)$ the probability density of $(x^m,w^m)$ and $u^\ast(x,w)$ the probability density of $(x,w)$ if $(x,v=w-x)$ is distributed according to density function $p_2$. One main reason to change $(x,v)$ to $(x,w)$ is that in~\cite{Cheng2017UnderdampedLM}, the authors showed that the map $(x_0,w_0)\rightarrow (x_t,w_t)$ induced from \eqref{eqn:underdampedLangevin} is a contracting map for for $t$. From~\cite{Cheng2017UnderdampedLM}, we also have:
\begin{equation}\label{trivialinequlaitySAGA}
 |x^m-x|^2+|v^m-v|^2\leq 4(|x^m-x|^2+|w^m-w|^2)\leq 16(|x^m-x|^2+|v^m-v|^2)\,\\
 \end{equation}
 and
 \begin{equation}\label{trivialinequlaitySAGA2}
  W^2_2(q^U_{m},p_2)\leq  4W^2_2(u_{m},u^*)\leq  16W^2_2(q^U_{m},p_2)\,.
\end{equation}

Similar to RCAD-O-LMC, define another trajectory of sampling by setting  $(\widetilde{x}^0,\wv^0)$ to be drawn from the distribution induced by $p_2$, and that
$\wx^m=\wrx_{hm},\ \wv^m=\wrv_{hm},\ \ww^m=\wx^m+\wv^m$ are samples from$\left(\wrx_t,\wrv_t\right)$ that satisfy
\begin{equation}\label{eqn:ULDSDE2SAGAstar}
\left\{\begin{aligned}
&\wrv_t=\wv_0e^{-2 t}-\gamma \int^t_0e^{-2(t-s)}\nabla f\left(\wrx_s\right)\rd s+\sqrt{4\gamma}e^{-2 t}\int^t_0e^{2 s}d\mathcal{B}_s \\
&\wrx_t=\wx_0+\int^t_0 \wrv_sds
\end{aligned}\right.\,,
\end{equation}
with the same Brownian motion as before. This leads to
\begin{equation}\label{eqn:ULDSDE2SAGAstar2}
\left\{\begin{aligned}
&\wv^{m+1}=\wv^me^{-2h}-\gamma \int^{(m+1)h}_{mh}e^{-2((m+1)h-s)}\nabla f(\wrx_s)\rd s+\sqrt{4\gamma}e^{-2 h}\int^{(m+1)h}_{mh}e^{2 s}d\mathcal{B}_s \\
&\wx^{m+1}=\wx^m+\int^{(m+1)h}_{mh} \wrv_sds
\end{aligned}\right.\,.
\end{equation}

Clearly $\left(\widetilde{X}_t,\widetilde{V}_t\right)$ can be seen as drawn from target distribution for all $t$, and initially we can pick $(\widetilde{x}^0,\widetilde{v}^0)$ such that
\[
W^2_2(q^U_0,p_2)=\EE\left(|x^0-\wx^0|^2+|v^0-\wv^0|^2\right)\,,\quad\text{and}\quad W^2_2(u_0,u^\ast)=\EE\left(|x^0-\wx^0|^2+|w^0-\ww^0|^2\right)\,.
\]

We then also define $\beta^m$
\[
\beta^0=\nabla f(\wx^0)
\]
and
\[
\beta^{m+1}_{r_m}=\partial_{r_m}f(\wx^m)\quad\text{and}\quad \beta^{m+1}_{i}=\beta^{m}_{i}\quad \text{if}\quad i\neq r_m\,,
\]

We will be showing the decay of the following Lyapunov function:
\begin{equation}\label{eqn:LynaULMC}
T^m\triangleq T^m_1+c_pT^m_2=\EE\left(|\wx^m-x^m|^2+|\ww^m-w^m|^2\right)+c_{p}\EE|\widetilde{g}^m-\beta^m|^2\,,
\end{equation}
where $c_{p}$ will be carefully chosen later.

The following lemma gives bounds for $T^{m+1}_1,T^{m+1}_2$ using $T^m_1,T^m_2$, and the proof of the theorem amounts to selecting the correct $c_p$.
\begin{lemma} Under conditions of Theorem \ref{thm:thmconvergenceULDSAGA}, we have
\begin{equation}\label{TM+11BOUNDULD}
\begin{aligned}
T^{m+1}_1< &D_1T^{m}_1+D_2T^{m}_2+D_3\,,
\end{aligned}
\end{equation}
\begin{equation}\label{TM+12BOUNDULD}
T^{m+1}_2\leq \frac{L^2}{d}T^m_1+\left(1-\frac{1}{d}\right)T^m_2\,,
\end{equation}
\end{lemma}
where
\[
D_1=1-h/(2\kappa )+244h^2d,\ D_2=84\gamma^2h^2d,\ D_3=672\gamma h^4d^4+30h^3d/\mu +260h^6d^2\,.
\]

\begin{proof}
The proof for bounding $T^m_2$ is the same as the one in Appendix~\ref{proofofRCAD-O-LMC} Lemma~\ref{lemmaTM2+1} and is omit from here. We only prove the first inequality.
\begin{itemize}[leftmargin=*]
\item \textbf{Step 1:} We firstly define $|\Delta^m|^2=|\ww^m-w^m|^2+|\wx^m-x^m|^2$, and compare \eqref{eqn:ULDSDE2SAGA} and \eqref{eqn:ULDSDE2SAGAstar2} for:
\[
\begin{aligned}
|\Delta^{m+1}|^2=&\left|
(\wv^m-v^m)e^{-2h}+(\wx^m-x^m)+\int^{(m+1)h}_{mh}\wrv_s-\rv_s\rd s\right.\\
&-\gamma\int^{(m+1)h}_{mh}e^{-2((m+1)h-s)}\left[\nabla f\left(\wrx_s\right)-\nabla f(x^m)\right]\rd s\\
&\left.+\gamma\int^{(m+1)h}_{mh}e^{-2((m+1)h-s)}E^m\rd s
\right|^2\\
&+\left|
(\wx^m-x^m)+\int^{(m+1)h}_{mh}\wrv_s-\rv_s\rd s
\right|^2\\
=&\left|\mathrm{J}^m_1\right|^2+\left|\mathrm{J}^m_2\right|^2=\left|\mathrm{J}^{r,m}_1+\mathrm{J}^{E,m}_1\right|^2+\left|\mathrm{J}^m_2\right|^2\,,
\end{aligned}
\]
where we denote
\[
\begin{aligned}
\mathrm{J}^{r,m}_1&=(\wv^m-v^m)e^{-2h}+(\wx^m-x^m)+\int^{(m+1)h}_{mh}\wrv_s-\rv_s\rd s\\
&-\gamma\int^{(m+1)h}_{mh}e^{-2((m+1)h-s)}\left[\nabla f\left(\wrx_s\right)-\nabla f(x^m)\right]\rd s\,
\end{aligned}
\]
and
\[
\mathrm{J}^{E,m}_1=\gamma\int^{(m+1)h}_{mh}e^{-2((m+1)h-s)}E^m\,.
\]
To control $\mathrm{J}^m_1$, we realize that $\mathrm{J}^{E,m}_1$ term, produced by $E^m$, is not perpendicular to the rest of the terms, namely $\mathrm{J}^{r,m}_1$, and it will lead to a lot of cross terms. We thus replace it by $\mathrm{J}^{\widetilde{E},m}_1$ induced by $\widetilde{E}^m$. This allows us to eliminate all cross terms. Since $E^m-\widetilde{E}^m$ is small, such replacement brings only small perturbation. In particular, with Young's inequality:
\begin{equation}\label{Jm1first}
\begin{aligned}
\EE\left|\mathrm{J}^m_1\right|^2\leq &(1+h^2)\EE\left|\mathrm{J}^m_1+\mathrm{J}^{\widetilde{E},m}_1-\mathrm{J}^{E,m}_1\right|^2+(1+1/h^2)\EE\left|\mathrm{J}^{\widetilde{E},m}_1-\mathrm{J}^{E,m}_1\right|^2\\
\leq &(1+h^2)\EE\left|\mathrm{J}^m_1+\mathrm{J}^{\widetilde{E},m}_1-\mathrm{J}^{E,m}_1\right|^2+\gamma^2(h^2+1)(2L^2\eta^2d+8L^2\eta^2d^2)
\end{aligned}\,,
\end{equation}
where we use the smallness of $E^m-\tilde{E}^m$ in \eqref{errorinG1}. The first term of \eqref{Jm1first} can be separated into three terms:
\[
\begin{aligned}
&\EE\left|\mathrm{J}^m_1+\mathrm{J}^{\widetilde{E},m}_1-\mathrm{J}^{E,m}_1\right|^2=\EE \left|\mathrm{J}^{r,m}_1+\mathrm{J}^{\widetilde{E},m}_1\right|^2\\
=&\EE\left|\mathrm{J}^{r,m}_1\right|^2+\EE\left|\mathrm{J}^{\widetilde{E},m}_1\right|^2+2\EE\left\langle\mathrm{J}^{r,m}_1,\mathrm{J}^{\widetilde{E},m}_1\right\rangle\\
\end{aligned}\,.
\]
Firstly note that
\[
\EE\left|\mathrm{J}^{\widetilde{E},m}_1\right|^2\leq \gamma^2h^2\EE\left|\widetilde{E}^m\right|^2\,.
\]
And to bound the third term, note that
\[
\EE\left\langle \mathrm{J}^{r,m}_1\,,\mathrm{J}^{\widetilde{E},m}_1\right\rangle
=\EE\left\langle \int^{(m+1)h}_{mh}\wrv_s-\rv_s\rd s\,,\mathrm{J}^{\widetilde{E},m}_1\right\rangle
\]
due to the fact that
\begin{equation}\label{eqn:rmindependence}
\EE\langle A,\tilde{E}^m\rangle = \EE\langle A,\EE_{r_m}\tilde{E}^m\rangle = 0
\end{equation}
for all $A$ that has no $r_m$ dependence. To further bound this term, we plug in the definition and have:
\[
\begin{aligned}
&2\EE\left\langle\int^{(m+1)h}_{mh}\wrv_s-\rv_s\rd s,\gamma\int^{(m+1)h}_{mh}e^{-2((m+1)h-s)}\rd s\widetilde{E}^m\right\rangle\\
=&-2\EE\left\langle\int^{(m+1)h}_{mh}\rv_s\rd s,\gamma\int^{(m+1)h}_{mh}e^{-2((m+1)h-s)}\rd s\widetilde{E}^m\right\rangle\\
=&2\EE\left\langle\gamma\int^{(m+1)h}_{mh}\int^s_{mh} e^{-2(s-t)}\rd t\rd sE^m,\gamma\int^{(m+1)h}_{mh}e^{-2((m+1)h-s)}\rd s\widetilde{E}^m\right\rangle\\
\leq&\gamma^2h^3(3\EE\left|\widetilde{E}^m\right|^2+4L^2\eta^2d+16L^2\eta^2d^2)
\end{aligned}\,,
\]
where we used~\eqref{eqn:rmindependence} again in the first and second equalities and
\[\EE\left\langle E^m,\widetilde{E}^m\right\rangle\leq 3\EE|\widetilde{E}^m|^2+2\EE\left|\widetilde{E}^m-E^m\right|^2
\]
together with \eqref{errorinG1} in the last inequality.

In conclusion, we have
\begin{equation}\label{Deltam+1ULD}
\begin{aligned}
T^{m+1}_1=\EE\left|\Delta^{m+1}\right|^2\leq &(1+h^2)\EE\left|\mathrm{J}^{r,m}_1\right|^2+\left|\mathrm{J}^m_2\right|^2+\gamma^2(h^2+1)(2L^2\eta^2d+8L^2\eta^2d^2)\\
&+(1+h^2)\left(\gamma^2h^2\EE\left|\widetilde{E}^m\right|^2+\gamma^2h^3(3\EE\left|\widetilde{E}^m\right|^2+4L^2\eta^2d+16L^2\eta^2d^2)\right)
\end{aligned}\,.
\end{equation}

Using $\gamma L=1$, $h<1$, $\eta<h^3$, we have
\begin{equation}\label{eqn:ULDSAGADELTAM+1}
\begin{aligned}
T^{m+1}_1=\EE\left|\Delta^{m+1}\right|^2\leq &(1+h^2)\EE\left|\mathrm{J}^{r,m}_1\right|^2+\left|\mathrm{J}^m_2\right|^2+2\gamma^2(h^2+3h^3)\EE\left|\widetilde{E}^m\right|^2\\
&+60h^6d^2
\end{aligned}\,.
\end{equation}
\item \textbf{Step 2:} Now, we study first two terms in~\eqref{eqn:ULDSAGADELTAM+1}. We try to bound $(1+h^2)\EE\left|\mathrm{J}^{r,m}_1\right|^2+\left|\mathrm{J}^m_2\right|^2$ using $T^{m}_1$ and $\EE|\widetilde{E}^m|^2$. We first try to separate out ($x^m,\widetilde{x}^m,v^m,\widetilde{v}^m$) from $\mathrm{J}^{r,m}_1$ and $\mathrm{J}^m_2$. Denote
\begin{equation}\label{eqn:ASAGA}
\begin{aligned}
A^m=&(\wv^m-v^m)(h+e^{-2h})+(\wx^m-x^m)\\
&-\gamma\int^{(m+1)h}_{mh}e^{-2((m+1)h-s)}\left[\nabla f(\wx^m)-\nabla f(x^m)\right]\rd s\,,
\end{aligned}
\end{equation}
\begin{equation}\label{eqn:BSAGA}
\begin{aligned}
B^m=&\int^{(m+1)h}_{mh}\wrv_s-\rv_s-(\wv^m-v^m)\rd s\\
&-\gamma\int^{(m+1)h}_{mh}e^{-2((m+1)h-s)}\left[\nabla f\left(\wrx_s\right)-\nabla f(\wx^m)\right]\rd s
\end{aligned}\,,
\end{equation}
\begin{equation}\label{eqn:CSAGA}
C^m=(\wx^m-x^m)+\int^{(m+1)h}_{mh}\wv^m-v^m\rd s=(\wx^m-x^m)+h(\wv^m-v^m)\,,
\end{equation}
\begin{equation}\label{eqn:DSAGA}
D^m=\int^{(m+1)h}_{mh}\wrv_s-\rv_s-(\wv^m-v^m)\rd s\,,
\end{equation}
then we have
\[
\mathrm{J}^{r,m}_1=A^m+B^m,\quad \mathrm{J}^m_2=C^m+D^m\,.
\]
By Young's inequality, we have
\begin{equation}\label{eqn:ABCDSAGA}
\begin{aligned}
(1+h^2)\EE|\mathrm{J}^{r,m}_1|^2+\EE|\mathrm{J}^m_2|^2=&(1+h^2)\EE|A^m+B^m|^2+\EE|C^m+D^m|^2\\
\leq &(1+a)\left((1+h^2)\EE|A^m|^2+\EE|C^m|^2\right)\\
&+(1+1/a)((1+h^2)\EE|B^m|^2+\EE|D^m|^2)\,,
\end{aligned}
\end{equation}
where $a>0$ will be carefully chosen later. Now, the first term of \eqref{eqn:ABCDSAGA} only contains information from previous step, using $f$ is strongly convex, we can bound it using $\left|\Delta^m\right|^2$ (showed in Lemma \ref{lem:ACSAGA}). To bound the second term, we need to consider difference between $x,v$ at $t_{m+1}$ and $t_m$, which can be bounded by $|\Delta^m|^2$ and $|E^m|^2$ (showed in Lemma \ref{lem:BDSAGA}).

According to Lemma \ref{lem:BDSAGA}-\ref{lem:ACSAGA},
we first have
\begin{equation}\label{JRM1JM2}
\begin{aligned}
&(1+h^2)\EE|\mathrm{J}^{r,m}_1|^2+\EE|\mathrm{J}^m_2|^2\\
\leq &(1+a)\left[1-h/\kappa +Dh^2\right]T^m_1\\
&+(1+1/a)\left[80h^4T^m_1+5\gamma^2h^4\EE|E^m|^2+5\gamma h^4d\right]\\
= &C_1T^1_m+5(1+1/a)\gamma^2h^4\EE|E^m|^2+5(1+1/a)\gamma h^4d\,,
\end{aligned}
\end{equation}
where in the first inequality we use $1+h^2<2$ and
\[
C_1=(1+a)[1-h/\kappa +Dh^2]+80(1+1/a)h^4\,.
\]

Plug \eqref{JRM1JM2} in \eqref{eqn:ULDSAGADELTAM+1} and also replace $\EE(|E^m|^2)$ with Lemma \ref{lem:varianceofE} equation \eqref{eqn:varianceofEULMC2}, we have
\begin{equation}\label{eqn:ULDSAGADELTAM+12}
\begin{aligned}
T^{m+1}_1\leq &C_1T^m_1+\gamma^2\left[10(1+1/a)h^4+8h^2\right]\EE\left|\widetilde{E}^m\right|^2\\
&+100(1+1/a)h^{10}d^2+5(1+1/a)\gamma h^4d+60h^6d^2\,,
\end{aligned}
\end{equation}
where we use $\gamma L=1,\eta<h^3$ and $h<1$.

\item \textbf{Step 3:} To ensure the decay of $T^m_1$, we need to choose $a$ 
such that the coefficient in front of $T_1^m$ is strictly smaller than $1$. Noting in
\[
C_1=(1+a)[1-h/\kappa +Dh^2]+80(1+1/a)h^4
\]
the second term is of high order, while the first one is of $1-O(h)$ amplified by $1+a$, so it is possible to choose $a$ small enough to make the entire term $1-O(h)$. Indeed, since $h\leq \frac{1}{(1+D)\kappa }$, we have
\[
1-h/\kappa +Dh^2\leq 1-2h/(3\kappa )\,,
\]

and thus by setting $a$ so that
\[
1+a=\frac{1-h/(2\kappa )}{1-2h/(3\kappa )}\,.
\]
The entire coefficient is $1-h/2\kappa +480\kappa h^3$ and is smaller than $1$ for moderately small $h$. Moreover, due to the definition of $a$, we have
\[
1+1/a\leq 6\kappa /h\,,
\]
plugging the calculation in \eqref{eqn:ULDSAGADELTAM+12} we have
\begin{equation}\label{eqn:ULDSAGADELTAM+13}
\begin{aligned}
T^{m+1}_1\leq &\left\{1-h/(2\kappa )+480\kappa h^3\right\}T^m_1\\
&+\gamma^2\left[60\kappa h^3+8h^2\right]\EE\left|\widetilde{E}^m\right|^2\\
&+600\kappa h^9d^2+30\gamma \kappa h^3d+60h^6d^2\,.
\end{aligned}
\end{equation}

We further bound $\EE|\tilde{E}^m|^2$ by plugging in Lemma \ref{lem:varianceofE} equation \eqref{eqn:varianceofEULMC} and use $\gamma L=1,\kappa h<1\leq d,\gamma \kappa = 1/\mu $, we have
\begin{equation}\label{eqn:ULDSAGADELTAM+14}
\begin{aligned}
T^{m+1}_1\leq &\left\{1-h/(2\kappa )+480\kappa h^3\right\}T^m_1\\
&+84h^2d\EE|\wx^m-x^m|^2\\
&+28\gamma^2h^2(24Lh^2d^4+3d\EE\left|\beta^m-\widetilde g^m\right|^2)\\
&+600\kappa h^9d^2+30\gamma \kappa h^3d+60h^6d^2\\
< &\left\{1-h/(2\kappa )+244h^2d\right\}T^m_1\\
&+84\gamma^2h^2d\EE\left|\beta^m-\widetilde g^m\right|^2\\
&+672\gamma h^4d^4+30h^3d/\mu +260h^6d^2
\end{aligned}\,,
\end{equation}
where we use $\EE|\wx^m-x^m|^2\leq \EE|\Delta^m|^2=T^m_1$ and try to absorb small terms into large terms to simplify the formula:
\[
60\kappa h^3+8h^2<28h^2,\quad 600\kappa h^9d^2+60h^6d^2<260h^6d^2,
\]
and
\[
480\kappa h^3+84h^2d\leq 244h^2d,\quad 30\gamma \kappa h^3d=30h^3d/\mu 
\]
 This proves \eqref{TM+11BOUNDULD}.
\end{itemize}
\end{proof}

Now we are ready to prove Theorem \ref{thm:thmconvergenceULDSAGA} by adjusting $c_p$.

\begin{proof}[Proof of Theorem \ref{thm:thmconvergenceULDSAGA}]
Plug \eqref{TM+11BOUNDULD} and \eqref{TM+12BOUNDULD} into \eqref{eqn:LynaULMC}:
\[
T^{m+1}\leq \left\{D_1+\frac{c_pL^2}{d}\right\}T^m_1+\left(1-\frac{1}{d}+\frac{D_2}{c_p}\right)c_pT^m_2+D_3\,.
\]
Note that according to the definition $D_3$ is of $O(h^3)$, and $D_2$ is of $O(h^2)$ while $D_1\sim 1-O(h)$, so it makes sense to choose $c_p$ small enough so that the coefficient for $T^m_1$ keeps being of $1-O(h)$. Indeed, we let
\[
c_p=168\gamma^2h^2d^2\,,
\]
and will have
\begin{equation}\label{eqn:ULDFINAL1}
\begin{aligned}
T^{m+1}\leq &\left\{1-h/(2\kappa )+412h^2d\right\}T^m_1+\left(1-\frac{1}{2d}\right)T^m_2\\
&+672\gamma h^4d^4+30h^3d/\mu +260h^6d^2
\end{aligned}\,,
\end{equation}
where we use $\gamma L=1$.

Using \eqref{eqn:conditionuetaULDSAGA}, we can verify
\[
\max\{1-h/(2\kappa )+412h^2d,1-1/2d\}\leq 1-h/(4\kappa ).
\]
Plug into \eqref{eqn:ULDFINAL1}, we have
\[
T^{m+1}\leq (1-h/(4\kappa ))T^m+672\gamma h^4d^4+30h^3d/\mu +260h^6d^2\,,
\]
by induction
\[
\begin{aligned}
T^{m}&\leq (1-h/(4\kappa ))^mT^0+2688\gamma \kappa h^3d^4+120\kappa h^2d/\mu +1040\kappa h^5d^2\\
&\leq (1-h/(4\kappa ))^mT^0+2688h^3d^4/\mu +120\kappa h^2d/\mu +1040\kappa h^5d^2\\
\end{aligned}\,.
\]

Finally, consider 
\[
\begin{aligned}
T^{0}&=\EE|\wx^0-x^0|^2+\EE|\ww^0-\ww^0|^2+c_{p}\EE|g^0-\beta^0|^2\\
&=\EE|\wx^0-x^0|^2+\EE|\ww^0-\ww^0|^2+c_{p}\EE|\nabla f(x^0)-\nabla f(y^0)|^2\\
&\leq (1+c_pL^2)(\EE|\wx^0-x^0|^2+\EE|\ww^0-\ww^0|^2)\leq 2W^2_2(q^O_0,p)\,,
\end{aligned}
\]
where we use $168\gamma^2h^2d^2L^2<1$. Taking square root on each term and use \eqref{trivialinequlaitySAGA2}, we finally obtain \eqref{eqn:convergeULDSAGA}.
\end{proof}

\section{Calculation of $\EE\left|\widetilde{E}^m\right|^2$ for RCAD-O-LMC}\label{VarwidetileEX}

According to the definition of~\eqref{initialg0}-\eqref{updateg}:
\[
\EE_{r_m}\widetilde{g}^{m+1} = \widetilde{g}^m +\frac{1}{d}\left(\nabla f(x^m)-\widetilde{g}^m\right)\,,\quad\EE_{r_m}\left(\widetilde{g}^{m+1}-\widetilde{g}^m\right) = \frac{1}{d}\left(\nabla f(x^m)-\widetilde{g}^m\right)\,,
\]
and
\[
\EE_{r_m}\left|\widetilde{g}^{m+1}-\widetilde{g}^m\right|^2=\sum_i\EE_{r_m}(\widetilde{g}^{m+1}_i-\widetilde{g}^m_i)^2 = \frac{1}{d}\sum_i|\partial_if(x^m)-\widetilde{g}^m_i|^2 \,.
\]
Naturally
\[
\EE_{r_m}\widetilde{F}^m = \widetilde{g}^m + \left(\nabla f(x^m)-\widetilde{g}^m\right)=\nabla f(x^m)\,.
\]
Accordingly,
\begin{equation}\label{expectationofEx}
\EE_{r_m}\left(\widetilde{E}^m\right)=\nabla f(x^m) -\EE_{r_m}(\widetilde{F}^m)=\textbf{0}
\end{equation}
and
\begin{equation}\label{varofEX}
\begin{aligned}
\EE_{r_m}\left|\widetilde{E}^m\right|^2&=\sum^d_{i=1}\EE_{r_m}|\widetilde{E}^m_i|^2=\sum^d_{i=1}\EE_{r_m}
\left|\partial_i f(x^m)-\widetilde{g}^m_i-d\left(\widetilde{g}^{m+1}_i-\widetilde{g}^m_i\right)\right|^2\\
& = (d-1)|\nabla f(x^m)-\widetilde{g}^m|^2\,.
\end{aligned}\,.
\end{equation}

Taking the expectation over the random trajectory:
\[
\EE\left|\widetilde{E}^m\right|^2=\EE\left(\EE_{r_m}|\widetilde{E}^m|^2\right)<d\EE|\nabla f(x^m)-\widetilde g^m|^2\,.
\]

To analyze each entry of $\partial_i f(x^m)-g^m_i$, we note:
\begin{equation}\label{eqn:varwidetilde}
\left|\partial_if(x^m)-\widetilde g^m_i\right|^2\leq 3\left|\partial_if(x^m)-\partial_if(y^m)\right|^2+3\left|\partial_if(y^m)-\beta^m_i\right|^2+3\left|\beta^m_i-\widetilde g^m_i\right|^2\,.
\end{equation}

The first term, after taking expectation and summing over $i$, becomes
\begin{equation}\label{varianceterm1}
\begin{aligned}
3\EE|\nabla f(x^m)-\nabla f(y^m)|^2\leq 3L^2\EE|\Delta^m|^2 = 3L^2T_1^m\,.
\end{aligned}
\end{equation}
The last term, with the same procedure, becomes $3T^m_2$. They both will be left in the estimate. We now focus on giving an upper bound of the second term. To do so we adopt a technique from~\cite{pmlr-v80-chatterji18a,Dubey2016VarianceRI}. Define $p=1/d$, for fixed $m\geq 1$ and $1\leq i\leq d$, we have
\[
\mathbb{P}(\beta^m_i=\partial_if(y^0))=(1-p)^m+(1-p)^{m-1}p
\]
and
\[
\mathbb{P}(\beta^m_i=\partial_if(y^j))=(1-p)^{m-1-j}p,\quad 1\leq j\leq {m-1}
\]

\begin{equation}\label{varianceterm2}
\begin{aligned}
&\EE\sum^d_{i=1}|\partial_if(y^m)-\beta^m_i|^2=\sum^d_{i=1}\sum^{m-1}_{j=0}\EE(\EE(|\partial_if(y^m)-\beta^m_i|^2|\beta^m_i=\partial_if(y^j)))\mathbb{P}(\beta^m_i=\partial_if(y^j))\\
= &\sum^{m-1}_{j=0}\sum^d_{i=1}\EE(|\partial_i f(y^m)-\partial_i f(y^j)|^2)\mathbb{P}(\beta^m_i=\partial_if(y^j))\\
\leq^{(I)} &\sum^{m-1}_{j=0}\EE(|\nabla f(y^m)-\nabla f(y^j)|^2)\mathbb{P}(\beta^m_1=\partial_1f(y^j))\\
\leq &L^2\sum^{m-1}_{j=0}\EE(|y^m-y^j|^2)\mathbb{P}(\beta^m_1=\partial_1f(y^j))\\
\leq &L^2\sum^{m-1}_{j=0}\EE(|y^m-y^j|^2)(1-p)^{m-1-j}p\\
&+L^2\EE(|y^m-y^0|^2)(1-p)^{m}\\
\leq^{(II)} &L^2\sum^{m-1}_{j=0}\EE\left(\left|\int^{mh}_{jh}\nabla f(y_s)ds-\sqrt{2h}\sum^{m-1}_{i=j}\xi_i\right|^2\right)(1-p)^{m-1-j}p\\
&+L^2\EE\left(\left|\int^{mh}_{0}\nabla f(y_s)ds-\sqrt{2h}\sum^{m-1}_{i=0}\xi_i\right|^2\right)(1-p)^{m}\\
\leq^{(III)} &L^2\sum^{m-1}_{j=0}\left[2h^2(m-j)^2\EE_p|\nabla f(y)|^2+4hd(m-j)\right](1-p)^{m-1-j}p\\
&+L^2\left[2h^2m^2\EE_p|\nabla f(y)|^2+4hdm\right](1-p)^{m}\\
\leq^{(IV)} &2ph^2L^2\EE_p|\nabla f(y)|^2\left[\sum^{m}_{j=1}j^2(1-p)^{j-1}+m^2(1-p)^m/p\right]\\
&+4phL^2d\left[\sum^m_{j=1}j(1-p)^{j-1}+m(1-p)^m/p\right]\\
\leq^{(V)} &\frac{8h^2L^2\EE_p|\nabla f(y)|^2}{p^2}+\frac{8hL^2d}{p}\\
\leq^{(VI)} &8hL^2d^2\left[\frac{hL^2d}{\mu }+1\right]
\end{aligned}\,,
\end{equation}
where in (I) we use $\mathbb{P}(\beta^m_i=\partial_if(y^j))$ are same for different $i$, (II) comes from \eqref{eqn:yt},\eqref{eqn:Bt}, (III) comes from $y_t\sim p$ for any $t$, (IV) comes from changing of variable, in (V) we use the bound for terms in the bracket and in (VI) we use $\EE_p|x-x^*|^2\leq d/\mu $ according to Theorem D.1 in \cite{pmlr-v80-chatterji18a}, where $x^*$ is the maximum point of $f$.

In conclusion, we have
\begin{equation}\label{varianceestimationofExm}
\begin{aligned}
\EE\left|\widetilde{E}^m\right|^2\leq &3dL^2T_1^m+3dT^m_2+24hL^2d^3\left[\frac{hL^2d}{\mu }+1\right]\,.
\end{aligned}
\end{equation}

\section{Key lemma in proof of RCAD-U-LMC}
\begin{lemma}\label{lem:D1SAGA} Under conditions of Theorem \ref{thm:thmconvergenceULDSAGA}, $\left(\wrx_t,\wrv_t\right)$ are defined in \eqref{eqn:ULDSDE2SAGAstar}, we have
\begin{equation}\label{xboundSAGA}
\EE\int^{(m+1)h}_{mh}\left|\wrx_t-\wx^m\right|^2\rd t\leq \frac{h^3\gamma d}{3}
\end{equation}
and
\begin{equation}\label{vboundSAGA}
\begin{aligned}
\EE\int^{(m+1)h}_{mh}\left|\left(\wrv_t-\rv_t\right)-\left(\wv^m-v^m\right)\right|^2\rd t\leq &16h^3\EE|\Delta^m|^2+\gamma^2h^3\EE|E^m|^2+0.4\gamma h^5d\,,
\end{aligned}
\end{equation}
\end{lemma}
\begin{lemma}\label{lem:BDSAGA} Under conditions of Theorem \ref{thm:thmconvergenceULDSAGA}, and $B^m,D^m$ are defined in \eqref{eqn:BSAGA},\eqref{eqn:DSAGA}, we have
\begin{equation}\label{bound:BSAGA}
\EE|B^m|^2\leq 32h^4\EE|\Delta^m|^2+2\gamma^2 h^4\EE|E^m|^2+2\gamma h^4 d
\end{equation}
\begin{equation}\label{bound:DSAGA}
\EE|D^m|^2\leq 16h^4\EE|\Delta^m|^2+\gamma^2h^4\EE|E^m|^2+0.4\gamma h^6d
\end{equation}
\end{lemma}
\begin{lemma}\label{lem:ACSAGA}
Under conditions of Theorem \ref{thm:thmconvergenceULDSAGA}, and $A^m,C^m$ defined in \eqref{eqn:ASAGA},\eqref{eqn:CSAGA}, there exists a uniform constant $D$ such that
\begin{equation}\label{ACboundSAGA}
\EE((1+h^2)|A^m|^2+|C^m|^2)\leq \left[1-h/\kappa +Dh^2\right]\EE|\Delta^m|^2\,
\end{equation}
where $\kappa =L/\mu $ is the condition number of $f$.
\end{lemma}

\begin{lemma}\label{lem:varianceofE} Under conditions of Theorem \ref{thm:thmconvergenceULDSAGA}, we have estimation for approximation gradient
\begin{equation}\label{eqn:varianceofEULMC}
\EE|\widetilde{E}^m|^2\leq 3dL^2\EE|\wx^m-x^m|^2+24Lh^2d^4+3d\EE\left|\beta^m-\widetilde g^m\right|^2
\end{equation}
and
\begin{equation}\label{eqn:varianceofEULMC2}
\EE|E^m|^2\leq 2\EE|\widetilde{E}^m|^2+20L^2h^6d^2\,.
\end{equation}
\end{lemma}

We prove these four lemmas below.

\begin{proof}[Proof of Lemma \ref{lem:D1SAGA}]
First we prove \eqref{xboundSAGA}. According to \eqref{eqn:ULDSDE2SAGAstar}, we have
\begin{equation}
\begin{aligned}
\EE\int^{(m+1)h}_{mh}\left|\wrx_t-\wx^m\right|^2dt&=\EE\int^{(m+1)h}_{mh}\left|\int^{t}_{mh} \wrv_sds\right|^2dt\\
&\leq \int^{(m+1)h}_{mh}(t-mh)\int^t_{mh} \EE\left|\wrv_s\right|^2dsdt\\
&=\int |v|^2p_2(x,v)\rd x\rd v\int^{(m+1)h}_{mh}(t-mh)^2dt=\frac{h^3\gamma d}{3}\,,
\end{aligned}
\end{equation}
where in the first inequality we use H\"older's inequality, and for the second equality we use $p_2$ is a stationary distribution so that $\left(\wrx_t,\wrv_t\right)\sim p_2$ and $\wrv_t\sim \exp(-|v|^2/(2\gamma))$ for any $t$.

Second, to prove \eqref{vboundSAGA}, using \eqref{eqn:ULDSDE2SAGA},\eqref{eqn:ULDSDE2SAGAstar}, we first rewrite $\left(\wrv_t-\rv_t\right)-\left(\wv^m-v^m\right)$ as
\begin{equation}\label{SSAGA}
\begin{aligned}
\left(\wrv_t-\rv_t\right)-\left(\wv^m-v^m\right)=&\left(\wv^m-v^m\right)(e^{-2(t-mh)}-1)\\
&-\gamma\int^t_{mh}e^{-2(t-s)}\left[\nabla f(\wrx_{s})-\nabla f(x^m)\right]\rd s\\
&+\gamma\int^t_{mh}e^{-2(t-s)}\rd s E^m\\
=&\mathrm{I}(t)+\mathrm{II}(t)+\mathrm{III}(t)\,.
\end{aligned}
\end{equation}
for $mh\leq t\leq (m+1)h$. Then we bound each term seperately:
\begin{itemize}
\item 
\begin{equation}\label{S1SAGA}
\begin{aligned}
\EE\int^{(m+1)h}_{mh}\left|\mathrm{I}(t)\right|^2\rd t&\leq h\EE\int^{(m+1)h}_{mh} \left|\left(\wv^m-v^m\right)(e^{-2(t-mh)}-1)\right|^2\rd t\\
&\leq h\int^{(m+1)h}_{mh} (2(t-mh))^2\EE\left|\wv^m-v^m\right|^2\rd t\\
&\leq \frac{4h^3}{3}\EE\left|\wv^m-v^m\right|^2\,,
\end{aligned}
\end{equation}
where we use H\"older's inequality in the first inequality and $1-e^{-x}<x$ in the second inequality.

\item
\begin{equation}\label{S2SAGA}
\begin{aligned}
&\EE\int^{(m+1)h}_{mh}\left|\mathrm{II}(t)\right|^2\rd t\leq \gamma^2\EE\int^{(m+1)h}_{mh}\left|\int^t_{mh}e^{-2(t-s)}\left[\nabla f(\wrx_s)-\nabla f(x^m)\right]\rd s\right|^2\rd t\\
\leq &2\gamma^2\EE\int^{(m+1)h}_{mh}\left|\int^t_{mh}e^{-2(t-s)}\left[\nabla f(\wrx_s)-\nabla f(\wx^m)\right]\rd s\right|^2\rd t\\
&+2\gamma^2\EE\int^{(m+1)h}_{mh}\left|\int^t_{mh}e^{-2(t-s)}\left[\nabla f(\wx^m)-\nabla f(x^m)\right]\rd s\right|^2\rd t\\
\leq &2\gamma^2\int^{(m+1)h}_{mh}(t-mh)\EE\int^t_{mh}\left|\nabla f(\wrx_s)-\nabla f(\wx^m)\right|^2\rd s \rd t\\
&+2\gamma^2\int^{(m+1)h}_{mh}(t-mh)\EE\int^t_{mh}\left|\nabla f(\wx^m)-\nabla f(x^m)\right|^2\rd s \rd t\\
\leq &2\gamma^2 L^2\int^{(m+1)h}_{mh}(t-mh)\EE\int^t_{mh}\left|\wrx_s-\wx^m\right|^2\rd s \rd t\\
&+2\gamma^2 L^2\int^{(m+1)h}_{mh}(t-mh)\EE\int^t_{mh}\left|\wx^m-x^m\right|^2\rd s \rd t\\
\leq &2\gamma^3 L^2d\int^{(m+1)h}_{mh}\frac{(t-mh)^4}{3}\rd t+2\gamma^2 L^2\int^{(m+1)h}_{mh}(t-mh)^2\rd t\EE\left|\wx^m-x^m\right|^2\\
\leq &\frac{2\gamma^3 L^2h^5d}{15}+\frac{2\gamma^2 L^2h^3}{3}\EE\left|\wx^m-x^m\right|^2\,,
\end{aligned}
\end{equation}
where in the third inequality we use gradient of $f$ is $L$-Lipschitz function and we use \eqref{xboundSAGA} in the fourth inequality.

\item 
\begin{equation}\label{S3SAGA}
\begin{aligned}
\EE\int^{(m+1)h}_{mh}\left|\mathrm{III}(t)\right|^2\rd t&=\gamma^2\EE\int^{(m+1)h}_{mh}\left|\int^t_{mh}e^{-2(t-s)}\rd s E^m\right|^2\rd t\\
&\leq \gamma^2\int^{(m+1)h}_{mh}(t-mh)^2\rd t\EE(|E^m|^2)\\
&\leq \frac{\gamma^2h^3}{3}\EE(|E^m|^2)\,,
\end{aligned}
\end{equation}
\end{itemize}
Plug \eqref{S1SAGA},\eqref{S2SAGA},\eqref{S3SAGA} into \eqref{SSAGA} and using $\gamma L=1$, we have
\[
\begin{aligned}
&\EE\int^{(m+1)h}_{mh}\left|\left(\wrv_t-\rv_t\right)-\left(\wv^m-v^m\right)\right|^2\rd t\\
\leq &3\left(\EE\int^{(m+1)h}_{mh}\left|\mathrm{I}(t)\right|^2\rd t+\EE\int^{(m+1)h}_{mh}\left|\mathrm{II}(t)\right|^2\rd t+\EE\int^{(m+1)h}_{mh}\left|\mathrm{III}(t)\right|^2\rd t\right)\\
\leq & 4h^3\left(\EE\left|\wx^m-x^m\right|^2+\EE\left|\wv^m-v^m\right|^2\right)+\gamma^2h^3\EE(|E^m|^2)+0.4\gamma h^5d\,,
\end{aligned}
\]
using \eqref{trivialinequlaitySAGA}, we get the desired result.
\end{proof}

\begin{proof}[Proof of Lemma \ref{lem:BDSAGA}]
First, we seperate $B^m$ into two parts:
\[
\begin{aligned}
\EE|B^m|^2\leq &2\EE\left|\int^{(m+1)h}_{mh} \left(\wrv_t-\rv_t\right)-\left(\wv^m-v^m\right)\rd t\right|^2\\
&+2\EE\left|\gamma\int^{(m+1)h}_{mh}e^{-2((m+1)h-t)}\left[\nabla f(\wrx_t)-\nabla f(\wx^m)\right]\rd t\right|^2\,.\\
\end{aligned}
\]
And each terms can be bounded:
\begin{itemize}
\item
\begin{equation}\label{V1boundSAGA}
\begin{aligned} 
&\EE\left|\int^{(m+1)h}_{mh} \left(\wrv_t-\rv_t\right)-\left(\wv^m-v^m\right)\rd t\right|^2\\
\leq &h\EE\int^{(m+1)h}_{mh}\left|\left(\wrv_t-\rv_t\right)-\left(\wv^m-v^m\right)\right|^2\rd t\\
\leq &16h^4\EE|\Delta^m|^2+\gamma^2h^4\EE(|E^m|^2)+0.4\gamma h^6d\,,
\end{aligned}
\end{equation}
where we use Lemma \ref{lem:D1SAGA} \eqref{vboundSAGA} in the second inequality.
\item
\begin{equation}\label{V2boundSAGA}
\begin{aligned} 
&\EE\left|\gamma\int^{(m+1)h}_{mh}e^{-2((m+1)h-t)}\left[\nabla f(\wrx_t)-\nabla f(\wx^m)\right]\rd t\right|^2\\
\leq &h\gamma^2\EE\int^{(m+1)h}_{mh}\left|e^{-2((m+1)h-t)}\left[\nabla f(\wrx_t)-\nabla f(\wx^m)\right]\right|^2\rd t\\
\leq &h\gamma^2L^2\EE\int^{(m+1)h}_{mh}\left|\wrx_t-\wx^m\right|^2\rd t\\
\leq &\frac{h^4\gamma^3L^2 d}{3}\leq \frac{h^4\gamma d}{3}\,,
\end{aligned}
\end{equation}
where we use Lemma \ref{lem:D1SAGA} \eqref{xboundSAGA} and $\gamma L=1$ in the last two inequalities. 

\end{itemize}
Combine \eqref{V1boundSAGA},\eqref{V2boundSAGA} together, we finally have
\[
\EE|B|^2\leq 32h^4\EE|\Delta^m|^2+2\gamma^2 h^4\EE(|E^m|^2)+0.8 h^6\gamma d+2h^4\gamma d/3\,,
\]
which implies \eqref{bound:BSAGA} if we further use $h<1$.

Next, estimation of $\left(\EE|D|^2\right)^{1/2}$ is a direct result of \eqref{V1boundSAGA}.
\end{proof}

\begin{proof}[Proof of Lemma \ref{lem:ACSAGA}]
Let $\wx^m-x^m=a$ and $\ww^m-w^m=b$. First, by the mean-value theorem, there exists a matrix $H$ such that $\mu {I}_d\preceq H\preceq L{I}_d$ and
\[
\nabla f(\wx^m)-\nabla f(x^m)=Ha\,.
\]
By calculation, $\int^{(m+1)h}_{mh}e^{-2((m+1)h-t)}\rd t=\frac{1-e^{-2h}}{2}$ and 
\[
\begin{aligned}
A^m&=(h+e^{-2h})(\wv^m-v^m)+\left(I_{d}-\frac{(1-e^{-2h})}{2}\gamma H\right)(\wx^m-x^m)\\
&=\left(\left(1-h-e^{-2h}\right)I_{d}-\frac{(1-e^{-2h})}{2}\gamma H\right)a+(h+e^{-2h})b
\end{aligned}\,.
\]
\[
C^m=(1-h)a+hb\,.
\]
Since $\|\gamma H\|_2\leq 1$ and we also have following calculation
\[
h+e^{-2h}=h+e^{-2h}-1+1=1-h+O(h^2)\,,
\]
\[
1-h-e^{-2h}=h+O(h^2)\,,
\]
\[
1-e^{-2h}=2h+O(h^2)\,.
\]
If we further define matrix $\mathcal{M}_A$ and $\mathcal{M}_C$ such that
\[
|A^m|^2=\left(a,b\right)^\top \mathcal{M}_A\left(a,b\right)\,,\quad |C^m|^2=\left(a,b\right)^\top \mathcal{M}_C\left(a,b\right)\,,
\]
then, we have
\[
\left\|\mathcal{M}_A-\begin{bmatrix}
0 & hI_{d}-\gamma hH\\
hI_{d}-\gamma hH & (1-2h)I_{d}
\end{bmatrix}\right\|_2\leq D_1h^2\,,
\]
and
\[
 \left\|\mathcal{M}_B-\begin{bmatrix}
(1-2h)I_{d} & hI_{d}\\
hI_{d} & 0
\end{bmatrix}\right\|_2\leq D_1h^2\,,
\]
where $D_1$ is a uniform constant since $h<1/1648$ by \eqref{eqn:conditionuetaULDSAGA}. This further implies
\[
(1+h^2)|A^m|^2+|C^m|^2=\left(a,b\right)^\top\begin{bmatrix}
(1-2h)I_{d} & 2hI_{d}-\gamma hH\\
2hI_{d}-\gamma hH & (1-2h)I_{d}
\end{bmatrix}\left(a,b\right)+h^2\left(a,b\right)^\top Q\left(a,b\right)
\]
where $\|Q\|_2\leq D_2$ and $D_2$ is a uniform constant. Calculate the eigenvalue of the dominating matrix (first term), we need to solve
\[
\mathrm{det}\left\{(1-2h-\lambda)^2I_{d}-(2hI_{d}-\gamma hH)^2\right\}=0\,,
\]
which implies eigenvalues $\{\lambda_j\}^d_{j=1}$ solve
\[
(1-2h-\lambda_j)^2-(2h-\gamma h\Lambda_j)^2=0\,,
\]
where $\Lambda_j$ is $j$-th eigenvalue of $H$. Since $\gamma\Lambda_j\leq \gamma L=1$ and $h<1$, we have 
\[
\lambda_j\leq 1-\gamma\Lambda_jh\leq 1-\mu h\gamma=1-h/\kappa 
\]
for each $j=1,\dots,d$. This implies
\[
\left\|\begin{bmatrix}
(1-2h)I_{d} & 2hI_{d}-\gamma hH\\
2hI_{d}-\gamma hH & (1-2h)I_{d}
\end{bmatrix}\right\|_2\leq 1-h/\kappa \,,
\]
and
\[
(1+h^2)|A^m|^2+|C^m|^2\leq (1-h/\kappa +Dh^2)(|a|^2+|b|^2)\,,
\]
where $D$ is a uniform constant. Take expectation on both sides, we obtain \eqref{ACboundSAGA}.

\end{proof}
\begin{proof}[Proof of Lemma \ref{lem:varianceofE}]

The proof is mostly the same as that in the calculation in Appendix~\ref{VarwidetileEX}. Inequality~\eqref{eqn:varwidetilde} still holds true except the second term needs to be treated differently. Following the step in Appendix~\ref{VarwidetileEX}, we define $p=1/d$, and then for fixed $m\geq 1$ and $1\leq i\leq d$, we have
\[
\mathbb{P}(\beta^m_i=\partial_if(\wx^0))=(1-p)^m+(1-p)^{m-1}p\,,
\]
and
\[
\mathbb{P}(\beta^m_i=\partial_if(\wx^j))=(1-p)^{m-1-j}p,\quad 1\leq j\leq {m-1}\,.
\]

\begin{equation}
\begin{aligned}
&\EE\sum^d_{i=1}|\partial_if(\wx^m)-\beta^m_i|^2=\sum^d_{i=1}\sum^{m-1}_{j=0}\EE(\EE(|\partial_if(\wx^m)-\beta^m_i|^2|\beta^m_i=\partial_if(\wx^j)))\mathbb{P}(\beta^m_i=\partial_if(\wx^j))\\
\leq &\sum^{m-1}_{j=0}\sum^d_{i=1}\EE(|\partial_i f(\wx^m)-\partial_i f(\wx^j)|^2)\mathbb{P}(\beta^m_i=\partial_if(\wx^j))\\
\leq &\sum^{m-1}_{j=0}\EE(|\nabla f(\wx^m)-\nabla f(\wx^j)|^2)\mathbb{P}(\beta^m_1=\partial_1f(\wx^j))\\
\leq &L^2\sum^{m-1}_{j=0}\EE(|\wx^m-\wx^j|^2)\mathbb{P}(\beta^m_1=\partial_1f(\wx^j))\\
\leq &L^2\sum^{m-1}_{j=0}\EE(|\wx^m-\wx^j|^2)(1-p)^{m-1-j}p\\
&+L^2\EE(|\wx^m-\wx^0|^2)(1-p)^{m}\\
\leq^{(II)} &L^2\sum^{m-1}_{j=0}\EE\left(\left|\int^{mh}_{jh}\wrv_sds\right|^2\right)(1-p)^{m-1-j}p\\
&+L^2\EE\left(\left|\int^{mh}_{0}\wrv_sds\right|^2\right)(1-p)^{m}\\
\leq^{(III)} &L^2\sum^{m-1}_{j=0}\left[2h^2(m-j)^2\EE_{p_2}|\wrv|^2 \right](1-p)^{m-1-j}p\\
&+L^2\left[2h^2m^2\EE_{p_2}|\wrv|^2\right](1-p)^{m}\\
\leq^{(IV)} &2ph^2L^2\EE_{p_2}|\wrv|^2\left[\sum^{m}_{j=1}j^2(1-p)^{j-1}+m^2(1-p)^m/p\right]\\
\leq^{(V)} &\frac{8h^2L^2\EE_{p_2}|\wrv|^2}{p^2}\\
\leq^{(VI)} &8\gamma h^2L^2d^3=8h^2Ld^3\,,
\end{aligned}
\end{equation}
where (II) comes from \eqref{eqn:ULDSDE2SAGAstar}, (III) comes from $\left(\wrx_t,\wrv_t\right)\sim p_2$ for any $t$, (IV) comes from changing of variable, in (V) we use the bound for terms in the bracket and in (VI) we use $\EE_{p_2}|v|^2\leq \gamma d$. This inequality differ from the derivation in Appendix~\ref{VarwidetileEX} only through (II).

Next, to prove \eqref{eqn:varianceofEULMC2}, we only need to notice
\[
\EE|E^m|^2\leq 2\EE|\widetilde{E}^m|^2+2\EE|F^m-\widetilde{F}^m|^2\,,
\]
\eqref{errorinG} and $\eta<h^3$ and follow the same calculation as in done in Appendix~\ref{VarwidetileEX}.
\end{proof}
\end{appendix}

\end{document}